\documentclass{fairmeta}
\usepackage[utf8]{inputenc}
\usepackage{lmodern}
\usepackage{amsmath}
\usepackage{amsfonts}
\usepackage{amssymb}
\usepackage{amsthm}
\usepackage{mathtools}
\usepackage{booktabs}
\usepackage{longtable}
\usepackage{float}
\usepackage{multirow}
\usepackage{colortbl}
\usepackage{adjustbox}
\usepackage{array}
\usepackage{url}

\newtheorem{lemma}{Lemma}

\title{VTrajEmbed: Aligning Text with Video Object Trajectories via Spatiotemporal Embeddings}
\author[1,*]{Ruizhi Chen}
\author[2,*]{Zining Fan}
\author[3,*]{Yudi Guo}
\author[1,4]{Jiang Liu}
\author[1]{Zhiyu Lin}
\author[1]{Zhongjiang He}
\author[3,\dagger]{Shouqian Shi}
\author[1,\dagger]{Hao Sun}
\affiliation[1]{Xingchen AGI Lab, China Telecom Artificial Intelligence Technology (Beijing) Co., Ltd.}
\affiliation[2]{Peking University}
\affiliation[3]{Nanjing University}
\affiliation[4]{Xi'an Jiaotong University}
\contribution[*]{Equal contribution.}
\contribution[\dagger]{Corresponding authors.}
\metadata[Correspondence]{Shouqian Shi and Hao Sun}
\date{October 13, 2026}

\abstract{
Fine-grained alignment between text and video object trajectories is a core challenge in multimodal understanding. Existing embedding models excel at global retrieval, yet their single global representation entangles multiple objects and their spatiotemporal dynamics, limiting direct object-level grounding. To address this issue, we propose \textbf{VTrajEmbed}, an embedding model that decomposes a video into object trajectories and aligns each trajectory embedding with text for object-level grounding through retrieval. Our main contributions: (1) \textbf{Slow-Fast Trajectory Encoding}, a dual-branch architecture that establishes trajectory representations by densely sampling low-resolution frames for high-frequency motion cues and sparsely selecting high-resolution frames for scene context, aggregated by a shared LLM backbone; (2) \textbf{Mask-Guided Contrastive Learning}, a training objective built upon these representations that couples an intra-group loss and a batch-level bidirectional loss under a relation mask, excluding unverified relations and preventing confirmed matches from becoming false negatives; (3) \textbf{Text2Traj-Bench}, a trajectory-level text-video benchmark with model-assisted construction and human-verified query--trajectory relations. Experimental results show that VTrajEmbed-5B improves text-to-trajectory and trajectory-to-text R@1 over Qwen3-VL-Embedding-8B by \textbf{15.48} and \textbf{18.05} percentage points, respectively, and achieves 49.48\% average R@1 across five global retrieval benchmarks, compared with 42.86\% for Qwen3-VL-Embedding-8B.

}

\begin{document}

\maketitle

\section{Introduction}
\label{sec:introduction}

Multimodal embedding models have become a key technique for mapping data from different modalities into a unified semantic space. The contrastive vision-language pretraining paradigm~\cite{radford2021learning,Zhai_2023_ICCV} has learned high-quality alignment on large-scale image-text pairs, and video embedding models further incorporate temporal dynamics into unified representations~\cite{zhu2024languagebind,li2026qwen3}, achieving notable progress in video-level retrieval. However, this success relies on the premise of compressing an entire video into a single semantic unit, and when the task shifts from ``does this video match this sentence'' to ``which object in the scene corresponds to this description,'' the granularity limitation of embedding models becomes apparent.

Real-world applications such as autonomous driving, robotic manipulation, and digital content moderation~\cite{pmlr-v229-zitkovich23a,sima2024drivelm,li2026towards} demand object-level alignment. In the static-image domain, FG-CLIP~\cite{xie2025fg2} and ObjEmbed~\cite{fu2026objembed} achieve region-text matching through fine-grained annotation or object-level decomposition, yet these advances do not extend to video: general-purpose video embedding models typically compress multiple objects and their dynamics into a single global representation~\cite{wemm-embedding}, and training signals remain sample-level without object-level positive or negative examples. Video grounding tasks introduce spatiotemporal supervision~\cite{gu2025knowing,gao2026omniground} but output localized tubes rather than a shared text-trajectory embedding space. A concurrent work~\cite{chen2026travel} similarly leverages ego-trajectory similarity as training-only supervision, yet still produces single-vector representations without object-level alignment. How to equip embedding models with object-level alignment while keeping learned representations useful for broader video-level retrieval remains an open problem.

To address these challenges, we propose \textbf{VTrajEmbed}, an embedding model that uses \emph{object trajectories} as the basic alignment unit, mapping object-level visual representations and text into a shared semantic space.
Given a video and a user query, VTrajEmbed first uses a lightweight query parser to extract textual descriptions of the referenced objects, which serve as prompts for a segmentation-and-tracking model to generate candidate trajectories.
Each trajectory is then encoded through a Slow-Fast dual-branch architecture, where dense low-resolution observations capture object appearance and motion, while sparse high-resolution frames provide scene context.
A shared language-model backbone aggregates the trajectory features and separately encodes the original query, producing embeddings for similarity-based retrieval.
For training, a mask-guided contrastive objective learns from confirmed query-trajectory matches and mismatches while excluding unverified relations.
To provide relation supervision and evaluate object-level alignment, we construct \textbf{Text2Traj-Bench}, a model-assisted benchmark with human-verified query-trajectory relations.
Unlike query-conditioned proposal generation in deployment, Text2Traj-Bench generates candidates from fixed category prompts and freezes each video's candidate pool before query generation, ensuring that all trajectory-level evaluations use the same query-independent candidates while the associated verified relations define relevance for object-centric video retrieval.
Our main contributions: 

\begin{itemize}

\item A Slow-Fast trajectory encoding architecture that captures high-frequency motion features via low-resolution dense sampling and scene context via high-resolution sparse sampling.

\item A mask-guided contrastive learning objective that enables object-level training, coupling local discrimination and bidirectional batch retrieval under a query-trajectory relation mask.

\item A model-assisted, human-verified benchmark that couples within-video evaluation with dense within-video and sparse cross-video relation annotations.

\item VTrajEmbed-5B achieves 69.67\% text-to-trajectory and 66.34\% trajectory-to-text R@1, substantially outperforming Qwen3-VL-Embedding-8B (54.19\% and 48.29\%). In the global retrieval evaluation, its trajectory-derived representations average 49.48\% R@1 across five widely used datasets, versus 42.86\% for the same baseline.

\end{itemize}

\section{Related Work}
\label{sec:related_work}

\textbf{Multimodal Embedding Models.}
Multimodal embedding models align images, videos, and text in a shared semantic space for similarity-based retrieval.
CLIP~\cite{radford2021learning} and SigLIP~\cite{Zhai_2023_ICCV} established large-scale contrastive image-text pretraining for cross-modal alignment.
Recent advances include retrieval data synthesis with MegaPairs~\cite{zhou2024megapairs}, universal embeddings from Jina Embeddings v4~\cite{jinaembeddingsv4}, MM-Embed~\cite{lin2025mm}, GME~\cite{zhang2025bridging}, and WeMM-Embedding~\cite{wemm-embedding}, and CLIP enhancement through LLM2CLIP~\cite{huang2026llm2clip}.

For video-text retrieval, \citet{fang2023transferring} transfer image-CLIP through temporal difference and alignment modules, while MASCOT~\cite{fang2023mascot} uses semantic masking and completion to capture text-related visual details.
LanguageBind~\cite{zhu2024languagebind}, InternVideo2~\cite{wang2024internvideo2}, and Qwen3-VL-Embedding~\cite{li2026qwen3} further advance video-text representations.
These methods primarily support clip-level retrieval, whereas VTrajEmbed aligns individual object trajectories with text.

\textbf{Object-Level Multimodal Alignment.}
Image-based alignment has progressed from region-language pretraining to explicit object representations.
RegionCLIP~\cite{zhong2022regionclip} aligns region proposals with textual concepts, GLIP~\cite{li2022glip} casts detection and phrase grounding into a unified pretraining task, and OWL-ViT~\cite{minderer2022owlvit} adapts image-text pretraining to open-vocabulary detection.
KOSMOS-2~\cite{peng2024kosmos2} supports grounded localization, SigLIP 2~\cite{tschannen2025siglip} improves dense visual features, and Motto~\cite{gao2026motto} combines spatially grounded thought tokens with reasoning for free-form grounding.

GOAL~\cite{choi2025goal} aligns objects with long descriptions, PixCLIP~\cite{xiao2025pixclip} learns pixel-text alignment, VLM-FO1~\cite{liu2025vlmfo1} retrieves region-token features for fine-grained perception, and MMCS~\cite{xiang2026mmcs} interleaves visual objects with language for object-level supervision.
FG-CLIP~\cite{xie2025fg2} and ObjEmbed~\cite{fu2026objembed} are especially relevant: the former improves fine-grained vision-language alignment, while the latter learns static-image object embeddings.
VTrajEmbed instead aligns text with video trajectories, jointly encoding appearance, motion, and temporal context.

\textbf{Open-Vocabulary Object Tracking and Segmentation.}
Open-vocabulary perception provides candidates for object-centric video retrieval.
Grounding DINO~\cite{liu2024grounding}, YOLO-World~\cite{cheng2024yolo}, DINO-X~\cite{ren2024dino}, and LLMDet~\cite{fu2025llmdet} localize objects beyond closed training categories.
SAM 2~\cite{ravi2025sam} extends promptable segmentation to videos, and SAM 3~\cite{carion2026sam} adds concept-prompted detection, segmentation, and tracking.
OVTR~\cite{li2025ovtr}, ReferDINO~\cite{liang2025referdino}, and CovTrack~\cite{qian2025covtrack} address open-vocabulary tracking, referring video object segmentation, and continuous open-vocabulary tracking.
DEMR~\cite{huang2026demr} models uncertainty for query-conditioned temporal localization, whereas Human-Centric Video Grounding~\cite{tang2021human}, Knowing Your Target~\cite{gu2025knowing}, and OmniGround~\cite{gao2026omniground} study spatio-temporal grounding.

These systems output temporal segments, locations, masks, tubes, or tracks.
VTrajEmbed instead encodes each candidate trajectory into a reusable text-aligned embedding for bidirectional retrieval, separating candidate generation from semantic matching.

\section{Method}
\label{sec:method}
VTrajEmbed follows a proposal-and-ranking pipeline (Figure~\ref{fig:vtrajembed-framework}).
A lightweight query parser extracts the subject to prompt a frozen segmentation-and-tracking module for query-conditioned trajectory proposals.
Slow-Fast encoding integrates appearance, motion, and scene context for similarity ranking against the original query in a shared embedding space.
Mask-guided bidirectional contrastive learning uses verified query-trajectory relations to jointly optimize within-video discrimination and bidirectional retrieval.
Text2Traj-Bench provides supervision and controlled evaluation: query-independent pools frozen before query generation supply common trajectory-level candidates, while verified relations establish relevance for cross-video evaluation.

\begin{figure}[!t]
  \centering
  \includegraphics[width=0.92\linewidth]{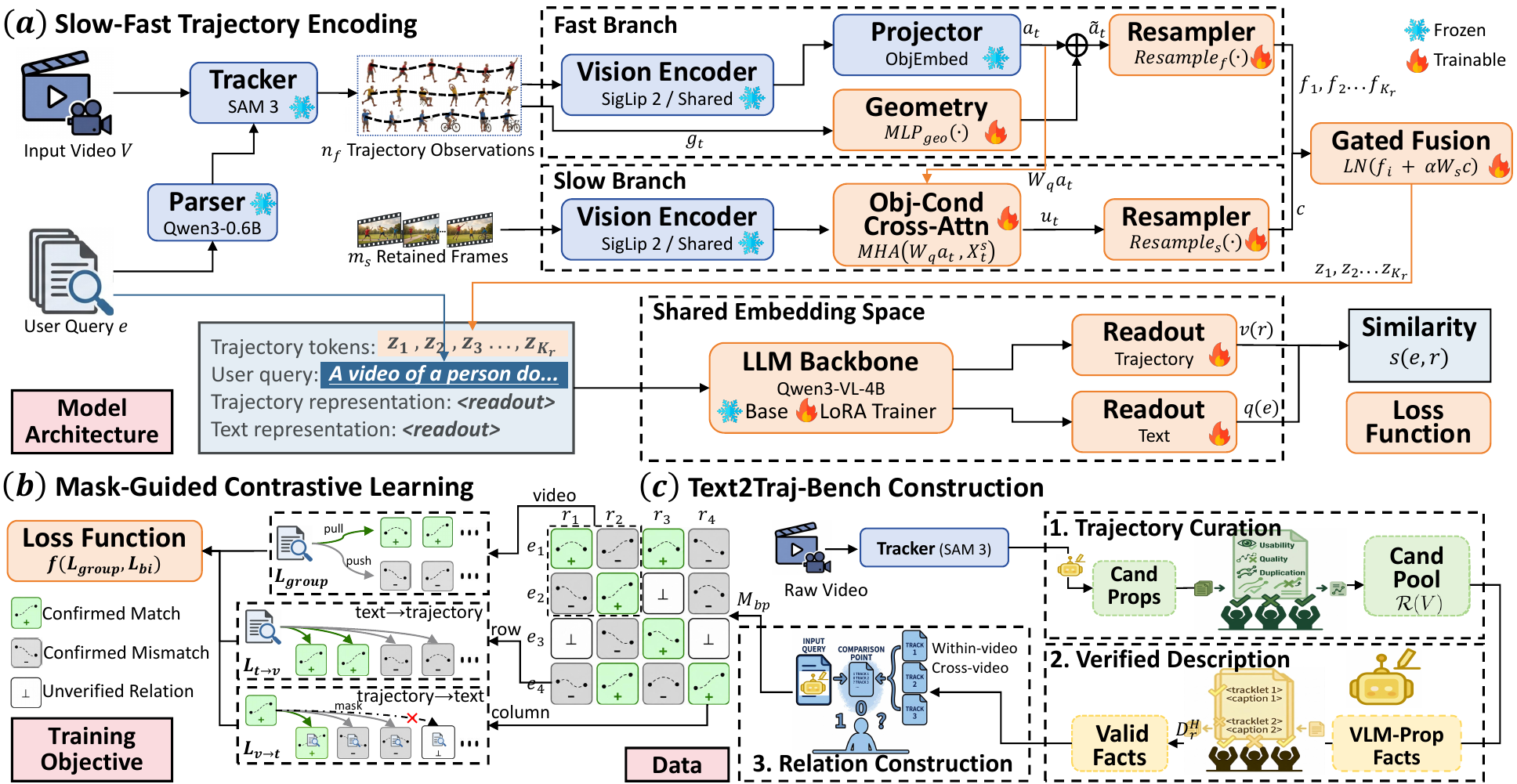}
  \caption{\textbf{Overview of VTrajEmbed.}
  \textbf{(a) Model architecture.} The Slow-Fast trajectory encoder combines dense low-resolution object observations with sparse high-resolution scene context and maps trajectories and text into a shared embedding space for similarity-based retrieval.
  \textbf{(b) Training objective.} Mask-guided contrastive learning jointly optimizes within-video trajectory discrimination and bidirectional trajectory-text retrieval using verified query-trajectory relations while excluding uncertain pairs.
  \textbf{(c) Data.} Text2Traj-Bench uses trajectory curation, verified description, and relation construction to produce fixed per-video candidate pools and human-verified annotations, including descriptions and within- and cross-video query-trajectory relations.}
  \label{fig:vtrajembed-framework}
\end{figure}

\subsection{Slow-Fast Trajectory Encoding}
\label{sec:slowfast}

Given a video \(V\) and query \(e\), VTrajEmbed retrieves the best-matching object trajectory by ranking candidates in a shared embedding space. 
A lightweight parser extracts the target description from \(e\) to prompt a frozen segmentation-and-tracking model, yielding candidates $\mathcal{R}(V,e)=\{r_1,\ldots,r_K\}$.
Each trajectory records the object's spatial extent and observation timestamps. 
VTrajEmbed independently embeds \(e\) and each candidate in a shared space and retrieves $\hat r=\arg\max_{r\in\mathcal{R}(V,e)}s(e,r)$ by trajectory-text similarity \(s(e,r)\).
Learning targets trajectory representation and cross-modal ranking; candidate generation remains frozen. 
For controlled evaluation, Text2Traj-Bench instead uses the query-independent pool \(\mathcal{R}_{\mathrm{bench}}(V)\), fixed before query generation and shared across models in trajectory-level evaluation; its verified relations establish relevance for object-centric video retrieval.

For each candidate, the Slow-Fast encoder combines dense low-resolution observations of fine-grained object appearance and motion with sparse high-resolution, object-conditioned scene context, as shown in Figure~\ref{fig:vtrajembed-framework}(a).
Gated fusion yields an ordered feature sequence that the shared language-model backbone aggregates into a trajectory embedding. 
A separate text-only pass through the same backbone embeds the original query.

\textbf{Fast branch: ordered object evidence.}
It encodes ordered object appearance and geometry from $n_f$ uniformly sampled valid observations, processing their full frames chronologically at low pixel budget $P_f$. 
The frozen vision encoder produces spatial visual features, from which a frozen ROI projector extracts an object appearance feature $a_t$ using the tracked box at time $t$. 
A trainable MLP embeds $g_t$, comprising normalized box center, size, area, inter-observation displacement, relative scale change, and original timestamp, to expose motion beyond appearance. 
Consecutive observations are grouped into temporal blocks $B_i$ and pooled by a shared attention resampler:
\begin{equation}
\label{eq:fast-blocks}
\widetilde a_t=a_t+\operatorname{MLP}_{\mathrm{geo}}(g_t),\qquad
f_i=\operatorname{Resample}_f([\widetilde a_t]_{t\in B_i}),\quad i=1,\ldots,K_r.
\end{equation}
The resulting $(f_1,\ldots,f_{K_r})$ compresses local observations while retaining block order for temporal modeling in the language backbone. Timestamps expose irregular observation gaps; uniform list sampling does not imply uniform physical-time spacing.

\textbf{Slow branch: object-conditioned scene context.}
The Slow branch complements the Fast sequence with sparse high-resolution scene context. 
It uniformly preselects $n_s$ frames from ordered Fast observations and scans them chronologically using unit-normalized semantic image embeddings from a frozen auxiliary encoder.
It retains the first frame and discards later frames only if cosine similarity to an already retained representative reaches $\tau$.
Representatives are never deleted and removed frames are not replenished. 
The retained set $S$ contains $m_s\leq n_s$ frames and a representative for every discarded frame; only these retained frames incur high-resolution vision encoding.
This pruning leaves the Fast observations and temporal blocks intact.

Object-conditioned cross-attention and gated fusion integrate the retained scene context into each valid Fast block.
The branches share the frozen vision encoder's weights but run separate forward passes at pixel budgets $P_f$ and $P_s>P_f$. At a retained time $t$, the matching Fast appearance feature $a_t$ queries the high-resolution full-frame token bank $X_t^s$ through cross-attention. 
A Slow attention resampler pools the resulting context features, and a learned residual gate injects their summary into each valid Fast block:
\begin{equation}
\label{eq:fusion}
u_t=\operatorname{MHA}(W_q a_t,X_t^s,X_t^s),\qquad
c=\operatorname{Resample}_s(\{u_t:t\in S\}),\qquad
z_i=\operatorname{LN}(f_i+\alpha W_s c).
\end{equation}
Here $W_q$ and $W_s$ are learned projections and $\alpha$ controls the context contribution. 
The shared $c$ summarizes surroundings across retained frames without frame-aligned interactions; event order remains in the Fast sequence and timestamps.
The design allocates temporal detail to object observations and spatial detail to contextual evidence, without assuming a strict frequency decomposition. Small objects and rapidly changing interactions may still require more spatial or temporal detail.

\textbf{Shared embedding space.}
The shared language backbone independently embeds trajectories and text for similarity-based retrieval. 
Projected, ordered $z_i$ fill the visual slots; a designated visual readout token yields $h_v$ and hence $v(r)=\operatorname{L2Norm}(W_vh_v)$. 
In a separate text-only pass, a query-specific text prompt contains the expression and a text readout, yielding $q(e)=\operatorname{L2Norm}(W_th_e)$ from text readout $h_e$ without active trajectory evidence.
Scores are $s(e,r)=\exp(\ell)q(e)^\top v(r)$, with learned logit scale $\ell$. 
Neither pass uses the other modality's input, allowing one query encoding and reusable trajectory embeddings for a fixed pool.
Appendix~\ref{app:encoding} gives the ROI implementation, geometry definition, and input templates.

\begin{lemma}[Task-relevant preservation under Slow-Fast sampling]
\label{lem:dualres}
Let $S_0$ contain $n_s$ preselected frames, and let semantic pruning retain $S\subseteq S_0$. Every frame $i$ is assigned to a retained representative $\pi(i)$ with unit semantic image embeddings satisfying $\bar g_i^\top\bar g_{\pi(i)}\geq\tau$. Suppose the object-conditioned context features obey $\|u_i-u_{\pi(i)}\|_2\leq L_u\|\bar g_i-\bar g_{\pi(i)}\|_2+\eta$, and the resampler $R$ is $K_R$-Lipschitz under maximum per-frame perturbations at fixed input length. Define $c_0=R([u_i]_{i\in S_0})$, $c=R([u_j]_{j\in S})$, and
\begin{equation}
\label{eq:multiplicity-error}
\varepsilon_{\mathrm{mult}}=
\left\|R([u_{\pi(i)}]_{i\in S_0})-R([u_j]_{j\in S})\right\|_2.
\end{equation}
Then $\|c-c_0\|_2\leq K_R\rho_\tau+\varepsilon_{\mathrm{mult}}$, where $\rho_\tau=L_u\sqrt{2(1-\tau)}+\eta$. If $c_D$ is a dense-context reference with $\|c_0-c_D\|_2\leq\varepsilon_{\mathrm{pre}}$, and the normalized trajectory encoder is $K_c$-Lipschitz in context with the Fast sequence fixed, then
\begin{equation}
\label{eq:slowfast-preservation}
\|v(r)-v_D(r)\|_2\leq
K_c\bigl(\varepsilon_{\mathrm{pre}}+K_R\rho_\tau+\varepsilon_{\mathrm{mult}}\bigr)
\equiv\delta_r.
\end{equation}
For a unit text embedding, a correct reference top-1 result with cosine margin $\gamma>2\max_r\delta_r$ is preserved.
\end{lemma}
The result explains when sparse high-resolution context can replace dense context while retaining the Fast observations. The semantic-to-context condition and resampler stability are explicit assumptions; deleting similar frames is not automatically lossless. Appendix~\ref{app:proof1} proves the bounds, quantifies the multiplicity term for attention pooling, and accounts for sampling and encoding cost.

\subsection{Mask-Guided Contrastive Learning}
\label{sec:mask}

Mask-guided contrastive learning jointly optimizes within-video discrimination and bidirectional batch retrieval. 
Query-centered trajectory groups from annotated and automatically expanded data supply these comparisons, as shown in Figure~\ref{fig:vtrajembed-framework}. 
Multiple query-trajectory pairs may match, while many others remain unverified.
A relation mask therefore includes confirmed positives and negatives and excludes uncertain pairs. 
Training updates lightweight spatiotemporal modules and low-rank language adapters, keeping the proposal pipeline and pretrained encoder weights frozen. 
Appendix~\ref{app:training} details the training data and relation annotations.

\textbf{Supervision and valid comparisons.}
Each training group $b$ contains a query $e_b$, positives $\mathcal P_b$, and confirmed within-video negatives $\mathcal N_b$. 
For annotated data, positives include the annotated tube, its stochastically perturbed variant, and an identity-checked tracker output when available. 
These augmented positives are training-only and leave fixed benchmark candidate pools unchanged. 
Expanded examples pair generated expressions with verified target trajectories. 
Low overlap or large description distance suggests negative candidates, not confirmed semantic mismatches.

\textbf{Within-video discrimination.}
It maximizes probability mass on each group's full positive set.
Let $s_{br}=s(e_b,r)$ and $\mathcal B_N$ index groups with negatives:
\begin{equation}
\label{eq:lgroup}
L_{\mathrm{group}}=-\frac{1}{|\mathcal B_N|}\sum_{b\in\mathcal B_N}
\log\frac{\sum_{r\in\mathcal P_b}\exp(s_{br})}
{\sum_{r\in\mathcal P_b\cup\mathcal N_b}\exp(s_{br})}.
\end{equation}
All retained negatives share one denominator per group; set $L_{\mathrm{group}}=0$ when $\mathcal B_N=\varnothing$.

\textbf{Relations across query groups.}
Across query groups, masked bidirectional retrieval compares queries with the batch's indexed positive-record list $\mathcal P=[(b,r):b\in\mathcal B,r\in\mathcal P_b]$. 
For query $b$ and record $p$, the relation mask $M_{bp}$ is $1$ for a confirmed match, $0$ for a confirmed mismatch, and $\bot$ otherwise. 
Queries may match records from other groups. 
Define $\mathcal J_b^+=\{p:M_{bp}=1\}$, $\mathcal J_b=\{p:M_{bp}\ne\bot\}$, and their column counterparts $\mathcal I_p^+$ and $\mathcal I_p$. With $s_{bp}=s(e_b,r_p)$,
\begin{align}
\label{eq:bidirectional}
L_{t\to v}&=-\frac{1}{|\mathcal B|}\sum_{b\in\mathcal B}
\log\frac{\sum_{p\in\mathcal J_b^+}\exp(s_{bp})}{\sum_{p\in\mathcal J_b}\exp(s_{bp})}, &
L_{v\to t}&=-\frac{1}{|\mathcal P|}\sum_{p\in\mathcal P}
\log\frac{\sum_{b\in\mathcal I_p^+}\exp(s_{bp})}{\sum_{b\in\mathcal I_p}\exp(s_{bp})}.
\end{align}
Each row and column has a positive by construction. Confirmed positives enter both numerator and denominator; unconfirmed pairs enter neither. We optimize
\begin{equation}
    L=\lambda_gL_{\mathrm{group}}+\lambda_{\mathrm{bi}}L_{\mathrm{bi}},
\end{equation}
where $L_{\mathrm{bi}}=(L_{t\to v}+L_{v\to t})/2$, with nonnegative weights $\lambda_g$ and $\lambda_{\mathrm{bi}}$.

The following lemma connects the joint objective to ranking-error control. 
Each comparison maximizes probability mass on its full positive set. 
\begin{lemma}[Task-aligned multi-positive ranking control]
\label{lem:dualoss}
For a local, row, or column comparison $A$ with finite scores and correct relation labels, let $A^+,A^-$ be its positive and negative sets, $m_A=|A^+|\geq1$, and $Z_A^{\pm}=\sum_{j\in A^{\pm}}e^{s_j}$. Let $E_A$ indicate that a negative scores at least as high as every positive, counting ties as errors and setting $E_A=0$ when $A^-$ is empty. Then
\begin{equation}
\label{eq:ranking-control}
E_A\leq\frac{\mathcal L_A}{\log(1+1/m_A)},\qquad
\mathcal L_A=\log(1+Z_A^-/Z_A^+).
\end{equation}
For families $x\in\{g,t,v\}$ of local, row, and column comparisons, let $E_x$ be their mean errors and $\omega_x=\log(1+1/m_x)$, where $m_x$ is the maximum positive count. Empty families contribute zero. The joint objective satisfies
\begin{equation}
\label{eq:joint-ranking-control}
\lambda_g \omega_gE_g+\frac{\lambda_{\mathrm{bi}}}{2}(\omega_tE_t+\omega_vE_v)\leq L.
\end{equation}
\end{lemma}
Thus $L$ bounds weighted errors across the three comparison families, even with overlapping score pairs. 
Including all confirmed matches and excluding uncertain relations aligns comparisons with the intended retrieval labels. 
Appendix~\ref{app:proof2} gives the proof, gradient consequences, and scope.

\subsection{Text2Traj-Bench: Human-Verified Trajectory Relations}
\label{sec:text2traj}

Text2Traj-Bench provides human-verified query-trajectory relations for learning (Section~\ref{sec:mask}) and controlled object-centric retrieval evaluation.
Its training split supports within-video discrimination and bidirectional batch retrieval, while its held-out splits support three controlled evaluation settings: within-video text-to-trajectory retrieval, trajectory-to-text retrieval, and object-centric video retrieval.
Construction follows three stages within each data split: trajectory curation, verified description, and relation construction, as illustrated in Figure~\ref{fig:vtrajembed-framework}(c).
Models propose annotations; human review confirms relations.
Source recordings are assigned to training, development, or test splits before construction, with duplicate clips, related clips, and derived records kept together.

\textbf{Trajectory curation.}
It establishes and freezes the query-independent candidate pool $\mathcal R_{\mathrm{bench}}(V)$ before query generation.
A frozen SAM~3 tracker generates trajectory proposals from fixed category prompts.
Three independent annotators assess trajectory usability and proposed duplicate pairs; retention and duplicate consolidation require at least two supporting votes.
Only confirmed redundant trajectories are consolidated, while distinct objects, nonredundant fragments, and unresolved identity pairs remain separate candidates.
All queries for $V$ and all models share this pool in trajectory-level evaluation; description or query rejection leaves it unchanged.

\textbf{Verified description.}
For curated trajectories, evidence-grounded query construction separates visual fact verification from query generation.
A VLM proposes appearance and motion descriptions, structured facts, and evidence references from ordered trajectory observations, supplemented by target-marked full frames for contextual claims.
Three annotators independently check the descriptions and facts against the video; the final version, including any minor evidence-supported corrections, requires at least two approvals.
An LLM then generates query $e$ solely from human-approved facts $D_r^H$ that support its required predicates, temporal scope, and relational qualifiers.

\textbf{Relation construction.}
It first determines query admission through human review of query faithfulness, naturalness, and every within-video query-trajectory pair against visual evidence, optionally using LLM pre-labels. 
Unresolved relations may temporarily receive $\bot$ during construction. 
A faithful, natural query with intended target trajectory $r^*$ is admitted only when
\begin{equation}
\label{eq:unique}
\bar Y(e,r^*)=1,\qquad
\bar Y(e,r)=0
\quad\forall r\in\mathcal R_{\mathrm{bench}}(V)\setminus\{r^*\}.
\end{equation}
Thus within-video relations are binary across all three splits.
Additional confirmed matches or uncertain relations block admission without changing the candidate pool.
Accepted training relations define $\mathcal P_b$ and $\mathcal N_b$ for $L_{\mathrm{group}}$, whereas held-out relations provide human-verified evaluation labels.

Following within-video verification, cross-video expansion constructs within-split relations for bidirectional training and cross-video evaluation. 
Verified descriptions guide mining of easy- and hard-negative candidates and potential matches; these categories guide sampling, not relation labels. 
An LLM compares query requirements with verified candidate facts and proposes a relation with supporting or conflicting predicates.
For training, humans individually verify potential matches and hard-negative candidates and review easy-negative candidates in stratified samples; unverified pairs remain $\bot$.
These cross-video and binary within-video relations populate training mask $M_{bp}$; $L_{\mathrm{bi}}$ excludes $\bot$ entries.
Optional calibrated automatic easy negatives are training-only and reported separately as pseudo-labels.
For development and test, every cross-video relation required by trajectory-to-text and object-centric video retrieval galleries is individually confirmed as a match or mismatch.

Fixed candidate pools and human-verified relevance labels define the benchmark's controlled evaluation protocol. 
Verified within-video relations label text-to-trajectory retrieval; verified cross-video relations define query-trajectory and query-video relevance for trajectory-to-text and object-centric video retrieval. 
All relations for development and test evaluation have binary labels; only training supervision retains unresolved pairs.
For trajectory-level evaluation, fixed pools isolate representation and ranking from proposal quality; associated verified relations define relevance for object-centric video retrieval.
Appendix~\ref{app:benchmark} details annotation, batching, auditing, and release protocols.

\section{Evaluation}
\label{sec:experiments}

\subsection{Experimental Setup}
\label{sec:experimental-setup}

\textbf{Implementation.}
VTrajEmbed trains only lightweight trajectory modules and LoRA adapters, keeping pretrained components frozen.
It uses the SigLIP2~\cite{tschannen2025siglip} vision encoder from Qwen3-VL-4B-Instruct, the ObjEmbed~\cite{fu2026objembed} object projector, and the language component of Qwen3-VL-4B~\cite{bai2025qwen3} as the shared trajectory-text backbone.
Training uses the joint objective in Section~\ref{sec:mask} for three epochs, with $\lambda_g=0.1$, $\lambda_{\mathrm{bi}}=1.0$, an initial learning rate of $5\times10^{-5}$, and an effective batch size of $32$.
Ablations share this configuration unless stated otherwise.
Architecture, optimization, and runtime details appear in Appendices~\ref{app:encoding} and~\ref{app:experimental-details}.

\textbf{Data and evaluation tasks.}
Text2Traj-Bench provides three retrieval settings using videos from HC-STVG~\cite{tang2021human}, Kinetics-700~\cite{carreira2019short}, MammAlps~\cite{gabeff2025mammalps} and other datasets.
Source recordings are split before construction; development data support model selection, and test data are reserved for final evaluation.
External-evaluation recordings from overlapping datasets are excluded from Text2Traj-Bench training and cross-video relation expansion.
Benchmark construction details appear in Appendix~\ref{app:benchmark}.
\emph{Text-to-Trajectory Retrieval (T2Traj)} ranks a video's fixed pool $\mathcal R_{\mathrm{bench}}(V)$ for a query;
\emph{Trajectory-to-Text Retrieval (Traj2T)} ranks all test queries for a trajectory;
and \emph{Object-Centric Video Retrieval (OCVR)} ranks videos for an object-centric query.

\textbf{Baselines.}
Baselines cover four multimodal embedding families: jina-embeddings-v4 (Jina-v4)~\cite{jinaembeddingsv4}, BGE-VL-v1.5-mmeb (BGE-VL)~\cite{zhou2024megapairs}, WeMM-Embedding (WeMM)~\cite{wemm-embedding}, and Qwen3-VL-Embedding (Qwen3E)~\cite{li2026qwen3}.
Scale suffixes identify model sizes, e.g., WeMM-4B and Qwen3E-8B.

\textbf{Evaluation protocol.}
Evaluation uses the same query-independent benchmark trajectories across models for T2Traj and Traj2T, with human-verified relevance labels for all three settings.
Baselines without native object-level inputs independently encode each trajectory as a video with its boxes overlaid on the corresponding frames.
For OCVR, conventional baselines use native whole-video embeddings, while VTrajEmbed scores videos by $S(e,V)=\max_{r\in\mathcal R_{\mathrm{bench}}(V)}s(e,r)$.
Relation labels and construction-time evidence annotations are withheld from model inputs.
All development and test gallery relations have confirmed binary relevance labels; no unverified pairs are scored.
All three settings report R@1 (primary), R@3, and mAP.
Global retrieval evaluates whether trajectory-derived representations remain useful after benchmark-specific alignment.
For each benchmark, VTrajEmbed uses multiple trajectory representations $v(r)$ per video for video-text alignment on the training split, followed by retrieval evaluation on the test split.

\begin{table}[!t]
\centering
\caption{\textbf{Object-centric retrieval results on Text2Traj-Bench.}
R@1, R@3, and mAP are reported in percent (higher is better).
Task definitions, model interfaces, and abbreviations follow Section~\ref{sec:experimental-setup}.
\textbf{Bold} marks column bests; parentheses show percentage-point changes from Qwen3E-8B.}
\label{tab:main_results}
\begingroup
\small
\setlength{\tabcolsep}{4pt}
\begin{adjustbox}{max width=\linewidth}
\begin{tabular}{lccc ccc ccc}
\toprule
\multirow{2}{*}{Model}
& \multicolumn{3}{c}{T2Traj}
& \multicolumn{3}{c}{Traj2T}
& \multicolumn{3}{c}{OCVR} \\
\cmidrule(lr){2-4} \cmidrule(lr){5-7} \cmidrule(lr){8-10}
& R@1 & R@3 & mAP
& R@1 & R@3 & mAP
& R@1 & R@3 & mAP \\
\midrule
Jina-v4~\cite{jinaembeddingsv4}
& 26.31 & 77.72 & 53.52
& 43.89 & 67.68 & 59.62
& 47.91 & 73.39 & 63.44 \\

BGE-VL~\cite{zhou2024megapairs}
& 60.66 & 92.73 & 76.79
& 39.56 & 57.31 & 53.93
& 60.85 & 84.08 & 73.88 \\

WeMM-2B~\cite{wemm-embedding}
& 45.10 & 86.33 & 66.31
& 20.98 & 40.43 & 37.14
& 42.60 & 60.61 & 56.10 \\

WeMM-4B~\cite{wemm-embedding}
& 46.61 & 87.20 & 67.59
& 21.46 & 38.83 & 37.20
& 44.13 & 62.38 & 57.41 \\

WeMM-9B~\cite{wemm-embedding}
& 46.76 & 86.81 & 67.56
& 18.09 & 33.92 & 33.52
& 45.26 & 62.86 & 58.36 \\

Qwen3E-2B~\cite{li2026qwen3}
& 44.79 & 86.56 & 66.01
& 43.66 & 57.54 & 51.53
& 65.86 & 84.58 & 76.35 \\

\rowcolor{gray!12}
Qwen3E-8B~\cite{li2026qwen3}
& 54.19 & 89.57 & 72.24
& 48.29 & 61.54 & 55.21
& \textbf{69.65} & 87.32 & 79.28 \\
\midrule

\rowcolor{blue!8}
VTrajEmbed-5B
& \textbf{69.67} & \textbf{98.51} & \textbf{83.72}
& \textbf{66.34} & \textbf{94.65} & \textbf{80.27}
& 68.04 & \textbf{95.21} & \textbf{81.43} \\

\rowcolor{blue!8}
&
{\footnotesize(15.48$\uparrow$)}
& {\footnotesize(8.94$\uparrow$)}
& {\footnotesize(11.48$\uparrow$)}
& {\footnotesize(18.05$\uparrow$)}
& {\footnotesize(33.11$\uparrow$)}
& {\footnotesize(25.06$\uparrow$)}
& {\footnotesize(1.61$\downarrow$)}
& {\footnotesize(7.89$\uparrow$)}
& {\footnotesize(2.15$\uparrow$)} \\
\bottomrule
\end{tabular}
\end{adjustbox}
\par
\endgroup

\vspace{0.8em}

\caption{\textbf{Global video retrieval results.}
R@1 is reported in percent (higher is better); Avg.\ is the unweighted mean over five datasets.
Evaluation settings and abbreviations follow Section~\ref{sec:experimental-setup}.
\textbf{Bold} marks column bests; parentheses show percentage-point changes from Qwen3E-8B.}
\label{tab:retrieval_results}
\begingroup
\small
\setlength{\tabcolsep}{4pt}
\begin{adjustbox}{max width=\linewidth}
\begin{tabular}{lcccccc}
\toprule
Method & Breakfast & HMDB51 & UCF101 & MSVD & MSR-VTT & Avg. R@1 \\
\midrule
WeMM-2B~\cite{wemm-embedding}
& 23.60
& 75.52
& 95.59
& 55.92
& 47.56
& 59.64 \\

WeMM-4B~\cite{wemm-embedding}
& 28.14
& \textbf{77.07}
& \textbf{96.01}
& 56.26
& \textbf{47.87}
& \textbf{61.07} \\

WeMM-9B~\cite{wemm-embedding}
& 27.99
& 76.68
& 95.49
& \textbf{56.52}
& 47.08
& 60.75 \\

Qwen3E-2B~\cite{li2026qwen3}
& 22.86
& 73.02
& 90.88
& 3.09
& 10.86
& 40.14 \\

\rowcolor{gray!12}
Qwen3E-8B~\cite{li2026qwen3}
& 22.85
& 76.26
& 94.97
& 5.89
& 14.33
& 42.86 \\
\midrule
\rowcolor{blue!8}
VTrajEmbed-5B
& \textbf{44.84}
& 75.12
& 94.55
& 19.64
& 13.26
& 49.48 \\
\rowcolor{blue!8}
\multicolumn{1}{c}{}
& \multicolumn{1}{c}{\footnotesize(21.99$\uparrow$)}
& \multicolumn{1}{c}{\footnotesize(1.14$\downarrow$)}
& \multicolumn{1}{c}{\footnotesize(0.42$\downarrow$)}
& \multicolumn{1}{c}{\footnotesize(13.75$\uparrow$)}
& \multicolumn{1}{c}{\footnotesize(1.07$\downarrow$)}
& \multicolumn{1}{c}{\footnotesize(6.62$\uparrow$)} \\
\bottomrule
\end{tabular}
\end{adjustbox}
\par
\endgroup
\end{table}

\subsection{Overall Performance}
\label{sec:overall-performance}

Table~\ref{tab:main_results} compares VTrajEmbed with existing multimodal embedding models on the three object-centric retrieval settings of Text2Traj-Bench.
Table~\ref{tab:retrieval_results} further evaluates whether the trajectory representations learned by VTrajEmbed can be effectively adapted to standard global video-text retrieval.
Overall, VTrajEmbed substantially improves fine-grained object-centric retrieval while its trajectory representations remain effective for broader global video-text retrieval after adaptation.

Compared with Qwen3E-8B, VTrajEmbed-5B improves T2Traj R@1 from 54.19\% to 69.67\% and Traj2T R@1 from 48.29\% to 66.34\%, corresponding to gains of 15.48 and 18.05 percentage points, respectively.
The improvement is even more pronounced under broader ranking metrics: Traj2T R@3 increases from 61.54\% to 94.65\%, and mAP from 55.21\% to 80.27\%.
After alignment on the training split of each global retrieval benchmark, VTrajEmbed achieves an average test-set R@1 of 49.48\%, compared with 42.86\% for Qwen3E-8B, with similar performance on HMDB51, UCF101, and MSR-VTT and substantially higher results on Breakfast and MSVD.
Taken together, these results show that VTrajEmbed substantially strengthens trajectory-level discrimination while its learned trajectory representations remain effectively adaptable to broader video-text retrieval.

\textbf{Trajectory-level modeling substantially improves fine-grained object discrimination.}
On T2Traj, VTrajEmbed achieves 69.67\% R@1, 98.51\% R@3, and 83.72\% mAP.
Compared with Qwen3E-8B, these results correspond to improvements of 15.48, 8.94, and 11.48 percentage points, respectively.
VTrajEmbed also outperforms BGE-VL, the strongest competing baseline on T2Traj, by 9.01 points in R@1, 5.78 points in R@3, and 6.93 points in mAP.
The gain is particularly meaningful because all methods rank exactly the same query-independent trajectory pools; therefore, it cannot be attributed to improved proposal recall or access to different candidate objects.
Instead, the results directly reflect a stronger ability to distinguish the target trajectory from visually or semantically competing objects within the same video.
This behavior is consistent with the Slow-Fast trajectory encoder: the Fast branch preserves ordered object appearance, geometry, and motion evidence, whereas the Slow branch supplements these observations with sparse high-resolution, object-conditioned scene context.

\textbf{Bidirectional supervision produces larger gains when retrieving text from trajectories.}
The advantage of VTrajEmbed becomes even more pronounced on Traj2T, where it reaches 66.34\% R@1, 94.65\% R@3, and 80.27\% mAP.
Relative to Qwen3E-8B, the improvements are 18.05, 33.11, and 25.06 percentage points, respectively.
The gain is therefore not restricted to identifying a target object from a small within-video candidate set: a trajectory representation also retrieves its corresponding descriptions more reliably from the substantially broader test-text gallery.
This observation provides empirical support for the bidirectional component of our training objective.
In addition to the within-video group loss that explicitly trains object discrimination, $L_{\mathrm{bi}}$ performs both text-to-trajectory and trajectory-to-text comparisons over verified cross-group relations.
The particularly large improvements in Traj2T R@3 and mAP are consistent with this design, suggesting that the learned trajectory embedding captures semantics that remain discriminative beyond its source-video candidate pool.

\textbf{Object-level alignment remains effective when aggregated to video-level retrieval.}
In OCVR, conventional embedding baselines rank videos using their native whole-video embeddings, whereas VTrajEmbed derives each video score from its highest text-trajectory similarity.
Under this trajectory-centric scoring interface, VTrajEmbed obtains 68.04\% R@1, 95.21\% R@3, and 81.43\% mAP.
Compared with Qwen3E-8B, R@3 and mAP increase by 7.89 and 2.15 percentage points, respectively, although R@1 decreases slightly by 1.61 points.
The combination of substantially higher R@3 and mAP with a small R@1 decrease suggests that trajectory-level matching more reliably places relevant videos among the highest-ranked candidates, while the ordering of the first-ranked video remains less stable.
This result is also consistent with the OCVR formulation $S(e,V)=\max_{r\in\mathcal R_{\mathrm{bench}}(V)}s(e,r)$: trajectory-level correspondence can support cross-video retrieval without introducing a separately trained video-level scoring module.
At the same time, the remaining R@1 gap indicates that aggregating independently scored object trajectories by a maximum operation does not fully resolve video-level ranking.

\textbf{Trajectory representations transfer effectively to global video retrieval.}
We further evaluate whether VTrajEmbed's learned trajectory representations support conventional global video-text retrieval after task-specific alignment.
For each benchmark, multiple trajectory representations \(v(r)\) from each video are used for video-text alignment on the training split, with retrieval evaluated on the test split.
After benchmark-specific adaptation, VTrajEmbed achieves an average R@1 of 49.48\%, compared with 42.86\% for Qwen3E-8B, despite using a smaller 4B language backbone.
Performance varies across datasets:
VTrajEmbed substantially exceeds Qwen3E-8B on Breakfast (44.84\% vs.\ 22.85\%, $+21.99$ points) and MSVD (19.64\% vs.\ 5.89\%, $+13.75$ points), while remaining close on HMDB51 (75.12\% vs.\ 76.26\%), UCF101 (94.55\% vs.\ 94.97\%), and MSR-VTT (13.26\% vs.\ 14.33\%).
The latter three results trail Qwen3E-8B by only 0.42--1.14 percentage points.
These results indicate that trajectory-level semantics remain useful beyond object-centric retrieval and can be effectively adapted to global video-text alignment.
Together with the substantially stronger T2Traj and Traj2T results, these findings suggest that VTrajEmbed improves fine-grained object discrimination without restricting its representations to trajectory-only retrieval.

\subsection{Ablation Study}
\label{sec:ablation}

Table~\ref{tab:ablation} examines the contributions of the Slow-Fast trajectory encoder and the two training objectives on T2Traj.
The full model performs best across all three metrics.
Removing either architectural branch or either loss term consistently degrades performance, indicating that both the trajectory representation and the joint supervision contribute to object-level retrieval.

\textbf{Dense Fast observations provide core trajectory evidence, complemented by sparse Slow context.}
Removing the Slow branch decreases R@1 from 69.67\% to 64.58\% and mAP from 83.72\% to 79.86\%.
With the Fast sequence unchanged, this degradation supports the contribution of sparse high-resolution, object-conditioned scene context beyond dense low-resolution trajectory observations.
In contrast, removing the Fast branch reduces R@1 to 41.08\%, R@3 to 87.01\%, and mAP to 64.94\%.
The substantially larger drop is consistent with the Fast branch preserving ordered object appearance, geometry, and temporal progression.

Denser Fast sampling further improves trajectory discrimination.
Reducing the maximum number of Fast observations from 32 to 16 lowers R@1 to 64.31\% and mAP to 79.81\%.
The smaller degradation relative to complete branch removal suggests that reduced sampling retains much of the Fast pathway's benefit but weakens fine-grained discrimination.

\textbf{Joint supervision improves T2Traj ranking over either objective alone.}
Removing $L_{\mathrm{bi}}$, leaving only the within-video group objective, reduces R@1 from 69.67\% to 49.29\% and mAP from 83.72\% to 71.10\%.
Conversely, removing $L_{\mathrm{group}}$ and training only with the bidirectional objective reduces R@1 to 48.06\% and mAP to 66.27\%.
The group-only variant retains higher R@3 than the bidirectional-only variant (93.55\% versus 81.45\%), consistent with $L_{\mathrm{group}}$ directly training within-video candidate discrimination.
However, its lower R@1 and mAP than the full model support combining local discrimination with relation-masked bidirectional retrieval.

\begin{table}[t]
\centering
\caption{\textbf{Ablation study of Slow-Fast encoding and training objectives.}
Other settings follow Section~\ref{sec:experimental-setup}.
Scores are percentages; parentheses show percentage-point drops from the full model.}
\label{tab:ablation}
\small
\setlength{\tabcolsep}{4pt}
\begin{tabular}{lccc}
\toprule
Variant & R@1 & R@3 & mAP \\
\midrule
\rowcolor{blue!8}
Full model
& \textbf{69.67}
& \textbf{98.51}
& \textbf{83.72} \\

w/o Slow branch
& 64.58{\small(5.09$\downarrow$)}
& 95.49{\small(3.02$\downarrow$)}
& 79.86{\small(3.86$\downarrow$)} \\

w/o Fast branch
& 41.08{\small(28.59$\downarrow$)}
& 87.01{\small(11.50$\downarrow$)}
& 64.94{\small(18.78$\downarrow$)} \\

Fast observations: 32 $\rightarrow$ 16
& 64.31{\small(5.36$\downarrow$)}
& 95.05{\small(3.46$\downarrow$)}
& 79.81{\small(3.91$\downarrow$)} \\

w/o $L_{\mathrm{bi}}$ (group only)
& 49.29{\small(20.38$\downarrow$)}
& 93.55{\small(4.96$\downarrow$)}
& 71.10{\small(12.62$\downarrow$)} \\

w/o $L_{\mathrm{group}}$ (bidirectional only)
& 48.06{\small(21.61$\downarrow$)}
& 81.45{\small(17.06$\downarrow$)}
& 66.27{\small(17.45$\downarrow$)} \\
\bottomrule
\end{tabular}
\end{table}

\section{Conclusion and Future Work}
\label{sec:conclusion}

We present VTrajEmbed, an embedding model that aligns text and video object trajectories in a shared semantic space.
Its Slow-Fast encoder combines dense low-resolution observations for appearance and motion with sparse high-resolution frames for scene context, while mask-guided contrastive learning couples within-video discrimination with bidirectional trajectory-text retrieval under verified relation supervision.
Text2Traj-Bench provides human-verified query-trajectory relations and fixed candidate pools for controlled trajectory-level evaluation, alongside verified cross-video relations for object-centric video retrieval.
VTrajEmbed-5B improves T2Traj and Traj2T R@1 over Qwen3E-8B by 15.48 and 18.05 percentage points, respectively.
After benchmark-specific adaptation, its trajectory-derived representations achieve 49.48\% average R@1 across five global retrieval datasets, compared with 42.86\% for Qwen3E-8B.
These results indicate finer object discrimination while preserving transferability to broader video-text retrieval.

Future work will extend fixed-pool evaluation to the full query-conditioned proposal-and-ranking pipeline, jointly measuring proposal recall, trajectory quality, and retrieval accuracy.
This will quantify how missing objects, fragmented trajectories, and identity switches affect downstream retrieval.
We will investigate learned aggregation of trajectory-level similarities to replace maximum-similarity scoring for more robust object-centric video retrieval.
Finally, we will expand Text2Traj-Bench with more object categories, video domains, naturally collected user queries, and verified cross-video relations to assess robustness and generalization across bidirectional and video-level retrieval.

\appendix
\section{Notation}
\label{app:notation}
Table~\ref{tab:notation} summarizes the mathematical notation used in the main text, including the lemmas and experimental captions. Related symbols are grouped by role; meanings that depend on context are stated explicitly.

\begingroup
\normalsize
\renewcommand{\arraystretch}{1.08}
\setlength{\tabcolsep}{6pt}
\setlength{\LTpre}{6pt}
\setlength{\LTpost}{6pt}
\begin{longtable}{@{}p{0.26\linewidth}p{\dimexpr0.74\linewidth-2\tabcolsep\relax}@{}}
\caption{Notation used in the main text.}\label{tab:notation}\\
\toprule
Symbol & Meaning \\
\midrule
\endfirsthead
\multicolumn{2}{@{}l}{Table~\thetable{} (continued)}\\
\toprule
Symbol & Meaning \\
\midrule
\endhead
\midrule
\multicolumn{2}{r@{}}{Continued on the next page}\\
\endfoot
\bottomrule
\endlastfoot

\multicolumn{2}{@{}l}{\textbf{Task and candidate ranking}}\\*
$V$, $e$ & \raggedright Input video and natural-language query, respectively.\tabularnewline
$r$, $\mathcal R(V)$ & \raggedright A candidate object trajectory and the fixed candidate pool for video $V$. Each trajectory stores boxes and original timestamps.\tabularnewline
$\hat r$, $r^*$ & \raggedright Highest-scoring retrieved trajectory and designated benchmark target, respectively.\tabularnewline
$s(e,r)$ & \raggedright Retrieval logit: $\exp(\ell)q(e)^\top v(r)$.\tabularnewline
$M$ & \raggedright Number of candidates in an evaluation pool; uniform random selection succeeds with probability $1/M$.\tabularnewline
$Y(e,r)$, $\widehat Y(e,r)$ & \raggedright True binary semantic-match relation and the verifier's estimate, respectively; the estimate can also be uncertain ($\bot$).\tabularnewline

\midrule
\multicolumn{2}{@{}l}{\textbf{Slow-Fast trajectory encoding (Section~\ref{sec:slowfast})}}\\*
$t$ & \raggedright Original observation time or frame timestamp; the subscript $t$ in a feature identifies that observation.\tabularnewline
$i$, $j$ & \raggedright Indices whose domains follow the surrounding expression: $i$ indexes Fast blocks or preselected frames, $j$ indexes retained frames in Lemma~\ref{lem:dualres}, and $j$ indexes comparison items in Lemma~\ref{lem:dualoss}.\tabularnewline
$n_f$, $n_s$, $m_s$ & \raggedright Numbers of valid Fast observations, preselected Slow frames, and retained Slow frames, respectively; $m_s\leq n_s\leq n_f$.\tabularnewline
$P_f$, $P_s$ & \raggedright Per-frame pixel budgets for the Fast and Slow branches; $P_s>P_f$.\tabularnewline
$a_t$, $g_t$ & \raggedright Object appearance feature extracted from the frozen target-ROI projector and its additional geometry/motion input, respectively. The symbol $a_t$ follows the main text and denotes the ROI appearance evidence used before the explicit geometry update, while $g_t$ contains normalized box geometry, displacement, relative scale change, and timestamp information.\tabularnewline
$\widetilde a_t$ & \raggedright Geometry-enhanced object feature: $a_t+\operatorname{MLP}_{\mathrm{geo}}(g_t)$.\tabularnewline
$B_i$, $K_r$ & \raggedright The $i$th consecutive block of Fast observations and the number of valid blocks for trajectory $r$, respectively.\tabularnewline
$m$, $K_{\max}$ &
\raggedright Fast temporal-block length and the fixed maximum number of Fast trajectory slots, respectively. The experimental configuration uses $m=4$ and $K_{\max}=8$.\tabularnewline
$f_i$ & \raggedright Fast feature obtained by attention-pooling the geometry-enhanced observations in block $B_i$.\tabularnewline
$S_0$, $S$ & \raggedright Preselected Slow frame set and the subset retained after semantic pruning, respectively.\tabularnewline
$\tau$ & \raggedright Cosine-similarity threshold for removing a preselected frame that is sufficiently similar to a retained representative.\tabularnewline
$X_t^s$ & \raggedright High-resolution full-frame visual token bank at retained time $t$; superscript $s$ denotes the Slow branch.\tabularnewline
$u_t$, $c$ & \raggedright Object-conditioned scene-context feature at time $t$ and the pooled context across retained Slow frames, respectively.\tabularnewline
$W_q$, $W_s$ & \raggedright Learned projections of the object appearance into the Slow attention query, and of pooled Slow context into the Fast residual, respectively.\tabularnewline
$\alpha$, $z_i$ & \raggedright Learned scalar context gate and the fused block feature $z_i=\operatorname{LN}(f_i+\alpha W_s c)$.\tabularnewline
$h_v$, $h_e$ & \raggedright Visual and text readout hidden states from their separate passes through the shared language backbone.\tabularnewline
$W_v$, $W_t$ & \raggedright Output projections mapping visual and text readouts into the shared embedding space, respectively. Here the subscript $t$ denotes text.\tabularnewline
$v(r)$, $q(e)$ & \raggedright Unit-normalized trajectory and text embeddings, respectively.\tabularnewline
$\ell$ & \raggedright Learned logit-scale parameter; $\exp(\ell)$ is the positive scale multiplying cosine similarity.\tabularnewline

\midrule
\multicolumn{2}{@{}l}{\textbf{Slow-context preservation (Lemma~\ref{lem:dualres})}}\\*
$\pi(i)$, $\bar g_i$ & \raggedright Retained representative assigned to preselected frame $i$, and the unit semantic image embedding of frame $i$, respectively. Thus $\bar g_{\pi(i)}$ is the representative's embedding; these embeddings differ from geometry inputs $g_t$.\tabularnewline
$u_i$, $u_{\pi(i)}$ & \raggedright Object-conditioned context features of preselected frame $i$ and its retained representative, using the same feature definition as $u_t$.\tabularnewline
$L_u$, $\eta$ & \raggedright Nonnegative coefficient and residual in the assumed bound relating semantic-image distance to context-feature distance.\tabularnewline
$R$, $K_R$ & \raggedright Slow attention resampler ($R=\operatorname{Resample}_s$) and its Lipschitz factor under maximum per-frame perturbations at fixed input length.\tabularnewline
$c_0$, $c_D$ & \raggedright Context pooled from all preselected frames before pruning, and the dense-context reference, respectively. The dense reference encodes all Fast observations at high resolution.\tabularnewline
$v_D(r)$ & \raggedright Normalized trajectory embedding obtained with $c_D$, using the same model parameters and Fast sequence as $v(r)$.\tabularnewline
$\varepsilon_{\mathrm{pre}}$ & \raggedright Upper bound on the preselection error $\|c_0-c_D\|_2$.\tabularnewline
$\rho_\tau$ & \raggedright Per-frame context-replacement bound $L_u\sqrt{2(1-\tau)}+\eta$.\tabularnewline
$\varepsilon_{\mathrm{mult}}$ & \raggedright Difference between pooling representatives repeated once per assigned preselected frame and pooling each retained frame once; see Eq.~\eqref{eq:multiplicity-error}. It accounts for the change in attention multiplicities.\tabularnewline
$K_c$ & \raggedright Lipschitz factor of the normalized trajectory encoder with respect to pooled context, holding the Fast sequence fixed.\tabularnewline
$\delta_r$ & \raggedright Upper bound on trajectory-embedding distortion: $K_c(\varepsilon_{\mathrm{pre}}+K_R\rho_\tau+\varepsilon_{\mathrm{mult}})$.\tabularnewline
$\gamma$ & \raggedright Dense-reference cosine margin of the correct target over its highest-scoring distractor. The lemma preserves this decision when $\gamma>2\max_r\delta_r$.\tabularnewline

\midrule
\multicolumn{2}{@{}l}{\textbf{Training groups, relations, and losses (Section~\ref{sec:mask})}}\\*
$b$, $e_b$, $\mathcal B$ & \raggedright Query-group index, the group's query, and the set of query-group indices in a training batch, respectively.\tabularnewline
$\mathcal P_b$, $\mathcal N_b$ & \raggedright Positive trajectories and confirmed within-video negatives for query group $b$.\tabularnewline
$\mathcal B_N$ & \raggedright Subset of batch groups containing at least one confirmed negative.\tabularnewline
$s_{br}$, $L_{\mathrm{group}}$ & \raggedright Within-group logit $s(e_b,r)$ and the multi-positive loss averaged over $\mathcal B_N$; the loss is zero if this set is empty.\tabularnewline
$\mathcal P$, $p$, $r_p$ & \raggedright Indexed list of all positive records in the batch, a record index, and the trajectory input of record $p$, respectively. Repeated records retain their multiplicities.\tabularnewline
$M_{bp}$, $\bot$ & \raggedright Query-record relation mask: $1$ for a confirmed match, $0$ for a confirmed mismatch, and $\bot$ for an unconfirmed relation excluded from the loss.\tabularnewline
$\mathcal J_b^+$, $\mathcal J_b$ & \raggedright Positive record indices $\{p:M_{bp}=1\}$ and all confirmed record indices $\{p:M_{bp}\ne\bot\}$ for query $b$.\tabularnewline
$\mathcal I_p^+$, $\mathcal I_p$ & \raggedright Positive query indices $\{b:M_{bp}=1\}$ and all confirmed query indices $\{b:M_{bp}\ne\bot\}$ for record $p$.\tabularnewline
$s_{bp}$ & \raggedright Batch logit $s(e_b,r_p)$ between query $b$ and trajectory record $p$.\tabularnewline
$L_{t\to v}$, $L_{v\to t}$ & \raggedright Masked multi-positive batch losses for text-to-trajectory and trajectory-to-text retrieval, respectively.\tabularnewline
$L_{\mathrm{bi}}$, $L$ & \raggedright Mean of the two batch-retrieval losses and the joint objective $L=\lambda_gL_{\mathrm{group}}+\lambda_{\mathrm{bi}}L_{\mathrm{bi}}$, respectively.\tabularnewline
$\lambda_g$, $\lambda_{\mathrm{bi}}$ & \raggedright Nonnegative weights of the within-video and bidirectional batch losses.\tabularnewline

\midrule
\multicolumn{2}{@{}l}{\textbf{Multi-positive ranking control (Lemma~\ref{lem:dualoss})}}\\*
$A$, $A^+$, $A^-$ & \raggedright One local, query-row, or trajectory-column comparison, together with its positive and negative item sets.\tabularnewline
$m_A$, $s_j$ & \raggedright Number of positive items in comparison $A$ and the finite logit of comparison item $j$, respectively.\tabularnewline
$Z_A^+$, $Z_A^-$ & \raggedright Sums of exponentiated logits over positive and negative items: $Z_A^{\pm}=\sum_{j\in A^{\pm}}\exp(s_j)$.\tabularnewline
$E_A$ & \raggedright Conservative top-1 error indicator: one when a negative scores at least as high as every positive, and zero otherwise (including an empty negative set). Ties count as errors.\tabularnewline
$\mathcal L_A$ & \raggedright Single-comparison multi-positive loss $\log(1+Z_A^-/Z_A^+)$.\tabularnewline
$x\in\{g,t,v\}$ & \raggedright Comparison-family label: $g$ denotes local within-video groups, $t$ query rows (text-to-trajectory), and $v$ trajectory columns (trajectory-to-text).\tabularnewline
$E_x$; $E_g,E_t,E_v$ & \raggedright Mean conservative ranking error for family $x$; the three explicit subscripts identify the corresponding local, row, and column errors.\tabularnewline
$m_x$ & \raggedright Maximum number of positives in a comparison from family $x$. Empty families contribute zero to the joint bound.\tabularnewline
$\omega_x$; $\omega_g,\omega_t,\omega_v$ & \raggedright Ranking-control coefficient $\log(1+1/m_x)$ for each nonempty comparison family.\tabularnewline

\midrule
\multicolumn{2}{@{}l}{\textbf{Experimental diagnostics and retrieval metrics}}\\*
$\overline m_s$ & \raggedright Mean number of retained Slow frames over the evaluated trajectories.\tabularnewline
$\overline\Delta_v$ & \raggedright Mean measured distortion $\|v(r)-v_D(r)\|_2$ over evaluated trajectories, distinct from the theoretical bound $\delta_r$.\tabularnewline
$\tau_-$, $\tau_\star$, $\tau_+$ & \raggedright Lower, reference, and higher pruning thresholds fixed on development data, with $\tau_-<\tau_\star<\tau_+$.\tabularnewline
R@1, R@3 & \raggedright Recall at ranks one and three. With one relevant trajectory per query, each is the fraction of queries whose target is within that rank cutoff.\tabularnewline
mAP & \raggedright Mean average precision over queries; the relevant-item sets must follow the stated retrieval protocol.\tabularnewline

\midrule
\multicolumn{2}{@{}l}{\textbf{Operators and notation conventions}}\\*
$\operatorname{MLP}_{\mathrm{geo}}$ & \raggedright Trainable multilayer perceptron mapping geometry/motion inputs to appearance-feature space.\tabularnewline
$\operatorname{Resample}_f$, $\operatorname{Resample}_s$ & \raggedright Attention resamplers for Fast temporal blocks and Slow context, respectively.\tabularnewline
$\operatorname{MHA}$, $\operatorname{LN}$ & \raggedright Multihead attention and layer normalization, respectively.\tabularnewline
$\operatorname{L2Norm}$ & \raggedright Normalization of a nonzero vector to unit Euclidean norm.\tabularnewline
$(\cdot)^\top$, $\|\cdot\|_2$, $|\cdot|$ & \raggedright Transpose, Euclidean norm, and set cardinality or indexed-list length, respectively. A list length counts repeated records.\tabularnewline
$[\cdot]$, $\{\cdot\}$ & \raggedright An indexed sequence/list and a set, respectively; repeated representatives in a list retain their multiplicities.\tabularnewline
$\exp$, $\log$, $\arg\max$ & \raggedright Natural exponential, natural logarithm, and an argument attaining the maximum. The exponential $e^{s_j}$ in Lemma~\ref{lem:dualoss} is unrelated to query $e$.\tabularnewline
\end{longtable}
\endgroup

\section{Experimental and Reproducibility Details}
\label{app:experimental-details}
This section consolidates the concrete model, optimization, batching, and runtime configurations used for the reported experiments.
Detailed encoder operations and training-data construction are described separately in Appendices~\ref{app:encoding} and~\ref{app:training}.

\subsection{Model and Encoder Configuration}
\label{app:model-config}
The frozen visual encoder is initialized from the vision tower of \texttt{Qwen3-VL-4B-Instruct}.
This vision tower is based on SigLIP2-Large and has undergone dynamic-resolution adaptation and multimodal continued training.
We extract intermediate spatial features from Transformer layers 12, 18, and 24, using one-based layer indices.
The frozen target-ROI projection module is initialized from \texttt{ObjEmbed-4B}.
The language backbone is initialized from the language component of \texttt{Qwen3-VL-4B-Instruct}.
The pretrained visual encoder, target-ROI projector, and base language-model parameters remain frozen during VTrajEmbed training.
For the Fast pathway, we sample at most $n_f^{\max}=32$ valid trajectory observations and use a temporal block length of $m=4$.
A trajectory therefore contains at most \begin{equation} K_{\max}=\left\lceil\frac{n_f^{\max}}{m}\right\rceil =\left\lceil\frac{32}{4}\right\rceil =8 \end{equation} Fast temporal blocks.
The three ROI feature scales are extracted from visual Transformer layers 12, 18, and 24.
Each original feature has width 1024 and is transformed by $2\times2$ spatial merging followed by an independent projection to a 2560-channel spatial feature map.
Before ROI extraction, the first two feature maps are upsampled by transposed convolutions with spatial factors of $4$ and $2$, producing 640 and 1280 channels, respectively; the third feature map retains 2560 channels.
A $7\times7$ ROI Align operation is applied independently at all three scales.
The resulting ROI features are concatenated along the channel dimension to obtain $640+1280+2560=4480$ channels and are linearly projected to 2560 channels at each spatial location.
The resulting $7\times7\times2560$ region representation is flattened to a 125,440-dimensional vector and passed through a $125440\rightarrow2560\rightarrow2560$ MLP with GELU activation to obtain the appearance vector.
This vector is the object appearance feature denoted by $a_t$ in the main text.
Bounding-box geometry is exposed explicitly through the additional Fast geometry input $g_t$ in Appendix~\ref{app:encoding}, rather than being treated as part of the symbol $a_t$.

\subsection{Optimization and Runtime Configuration}
\label{app:optimization-config}
We train VTrajEmbed using fused AdamW (\texttt{adamw\_torch\_fused}) with an initial learning rate of $5\times10^{-5}$, $\beta_1=0.9$, $\beta_2=0.999$, $\epsilon=10^{-8}$, and zero weight decay.
The learning rate follows a linear decay schedule from its initial value to zero, with no warm-up (\texttt{warmup\_steps}=0).
Each training process reads four expression-centered training samples per step.
We use one process per GPU on eight NVIDIA A800 GPUs with 80\,GB memory each and one gradient-accumulation step.
The effective global batch size is therefore \begin{equation} 4\times8\times1=32 \end{equation} expression samples per optimizer update.
Training is controlled by an epoch budget of three epochs; no positive \texttt{max\_steps} value overrides this budget.
The training random seed is 42.
The software environment uses PyTorch~2.10.0, torchvision~0.25.0, and Transformers~5.17.0.
Unless otherwise specified, architectural ablations use the same data split, optimizer, learning-rate schedule, batch configuration, training duration, and random seed as the full model.
Objective ablations modify only the corresponding loss terms.

\subsection{Expression-Group Construction and Noisy-GT Augmentation}
\label{app:augmentation-config}
Each expression-centered training sample contains four trajectory records.
One clean ground-truth trajectory and one confirmed negative are always retained.
The remaining two records are filled by either (i) one noisy-GT trajectory and one matched raw positive, or (ii) two matched raw positives when the noisy-GT variant is unavailable.
Thus, each expression group contains one clean GT, zero or one noisy GT, one or two matched raw positives, and one confirmed negative while maintaining a fixed group size of four trajectory records.
Noisy-GT augmentation is enabled by default.
The temporal augmentation first removes one contiguous segment comprising approximately 10\%-20\% of the trajectory observations and then independently drops each remaining observation with probability 0.1.
The augmented trajectory is required to retain at least one valid observation.
The clean ground-truth trajectory is maintained as a separate positive record and is not replaced by the noisy variant.
For spatial perturbation, the box center is independently displaced by up to $\pm5\%$ of the original box width and height along the horizontal and vertical axes, respectively.
The box width and height are independently scaled by factors in $[0.9, 1.1]$.
Perturbed boxes are clipped to the image boundaries.
If clipping produces a degenerate box, the original box for that observation is restored.

\section{Candidate Generation and Encoder Details}
\label{app:encoding}

\subsection{Candidate generation and task scope}
\label{app:candidates}
VTrajEmbed separates proposal generation from embedding-based ranking. A candidate $r=\{(t,b_t)\}$ stores nonempty bounding boxes $b_t$ in original-image coordinates and their physical frame indices $t$. It represents the observations associated with one tracker identity. ``Complete trajectory'' means retaining the available identity track rather than clipping it to the action interval; it does not imply that every physical appearance of the object was successfully tracked. VTrajEmbed neither predicts new boxes nor regresses event boundaries.

The general query-conditioned proposal interface uses a frozen Qwen3-0.6B parser to extract appearance entities, retaining attributes such as clothing and color while discarding action and interaction descriptions. For example, a query about a man in a white suit and a woman in a black dress walking together yields two appearance prompts. Each prompt drives the frozen SAM3 tracker, and the union of its outputs forms the proposal pool. An appearance-matching object is retained even when its behavior does not match the query. When several queries are processed together, their prompts and proposals can be deduplicated and shared. This mode conditions proposal recall on the query, although the embedding of an already specified trajectory remains query-independent.

Text2Traj-Bench generates candidate trajectories from fixed category prompts and freezes each video's candidate pool after trajectory-quality and redundancy review, before verified description and query generation.
Subsequent description or query rejection does not alter the frozen pool.
The current person-oriented description examples and source setting should not be interpreted as evidence of performance on every object category. Other categories can use the same construction procedure, but require their own prompts, quality checks, and validation. The benchmark does not rerun the appearance parser to prune candidates separately for each evaluated query.

SAM3 detects objects at anchor frames and propagates masks forward and backward over the video timeline. Bounding boxes are extracted from nonempty masks and mapped to the original image coordinates. The encoder consumes rectangular bbox ROIs, not segmentation masks. Empty masks provide no observation; propagation failures and identity switches are tracking errors, not guarantees resolved by the encoder. Both the parser and tracker remain frozen, and neither receives gradients from the embedding objective. Replacing the tracker changes candidate recall and quality, so results using different proposal pools are not directly comparable as ranking results.

\subsection{Sampling and temporal blocks}
Sort each trajectory's $n(r)\geq1$ observations by timestamp. The Fast branch selects $n_f(r)=\min(n(r),n_f^{\max})$ distinct observations uniformly over this ordered list. The reference maximum is $n_f^{\max}=32$. If multiple observations are selected, evenly spaced list indices include both endpoints; a single observation uses the first available index. Sampling is performed independently for each trajectory without modifying the stored observations.

The Slow branch first preselects $n_s(r)=\min(n_s^{\mathrm{budget}},n_f(r))$ distinct observations, where the configured budget is one to four. Preselection uses evenly spaced indices in the ordered Fast observations, including endpoints when more than one frame is selected and using the first observation when only one is selected. These are positions in the valid-observation list, not a claim of uniform physical-time spacing across occlusions. Let $S_0$ denote this initial set and let $\bar g_i$ be its unit-normalized global semantic-image features from the frozen auxiliary encoder, referred to below as CLIP-style features.

Next, scan $S_0$ chronologically and maintain a retained set $S$. Keep the first frame. For each subsequent frame $i$, find $j^*=\arg\max_{j\in S}\bar g_i^\top\bar g_j$, breaking ties by timestamp. If the similarity is at least a fixed threshold $\tau\in[-1,1]$, discard $i$ and assign $\pi(i)=j^*$; otherwise retain $i$ and set $\pi(i)=i$. Retained representatives are never deleted, and removed frames are not replenished. Therefore
\begin{equation}
\label{eq:slow-selection}
1\leq m_s(r)=|S|\leq n_s(r),\qquad
\bar g_i^\top\bar g_{\pi(i)}\geq\tau\quad(i\in S_0).
\end{equation}
This is semantic pruning of a preselected set, rather than selection by maximum semantic difference. The representative condition applies to the final retained set: a chain of pairwise merges that later deletes a representative would not provide the same guarantee. The auxiliary-encoder comparison occurs before high-resolution encoding; only retained frames enter the Slow vision encoder and temporal resampler. The Fast observations, ROI features, and temporal blocks are unchanged by this pruning.

The assignment defines cluster sizes $\kappa_j=|\{i\in S_0:\pi(i)=j\}|$, which may be recorded for auditing. The current attention resampler receives the retained features without cluster-size reweighting. These counts are used to analyze the resulting change in attention mass, not to introduce a new weighted pooling module. Fix the auxiliary encoder checkpoint, preprocessing, preselection budget, and threshold independently of test performance, and report the observed retained-frame counts. The threshold controls the semantic covering radius $\sqrt{2(1-\tau)}$; a larger threshold tightens this guarantee but does not imply a monotone retained count for every greedy run.

With the block length fixed to $m=4$, the valid block count is \begin{equation} K_r=\left\lceil\frac{n_f(r)}{4}\right\rceil. \end{equation}
Because the Fast pathway uses at most $n_f^{\max}=32$ observations, the fixed maximum slot count is \begin{equation} K_{\max}=\left\lceil\frac{32}{4}\right\rceil=8. \end{equation}
Each Fast temporal block therefore contains at most four valid observations.
The final block contains only its actual observations; short trajectories are never extended by copying frames.
Only the $K_r$ valid block features represent trajectory evidence at the language-model input; any implementation-level padding up to $K_{\max}$ is masked and does not represent additional observations.
The complete experimental configuration is summarized in Appendix~\ref{app:experimental-details}.

\subsection{Fast branch: ROI features and geometry}
The Fast pathway uses the frozen vision tower initialized from \texttt{Qwen3-VL-4B-Instruct}, which is based on SigLIP2-Large and has undergone dynamic-resolution adaptation and multimodal continued training.
Low-resolution full frames are processed with a per-frame pixel budget $P_f=65{,}536$.

For each input frame, we extract spatial Transformer features from layers 12, 18, and 24, using one-based layer indices.
The original feature width at each tap is 1024.
Each tapped feature is first subjected to $2\times2$ spatial merging and an independent projection, producing three spatial feature maps with 2560 channels while preserving their spatial organization.

The frozen target-ROI projection module is initialized from \texttt{ObjEmbed-4B}.
Before ROI extraction, the first two feature maps are upsampled by transposed convolutions with spatial factors of $4$ and $2$, reducing their channel widths to 640 and 1280, respectively, while the third feature map remains at 2560 channels.
For each tracked object box, we apply $7\times7$ ROI Align independently to all three scales and concatenate the resulting region features along the channel dimension, yielding \[ 640+1280+2560=4480 \] channels at each ROI grid location.
A position-wise linear projection maps this 4480-dimensional representation to 2560 dimensions.

The resulting $7\times7\times2560$ region tensor is flattened into a 125,440-dimensional vector and passed through a $125440\rightarrow2560\rightarrow2560$ MLP with GELU activation to produce an appearance vector.
This frozen 2560-dimensional appearance feature is denoted $a_t\in\mathbb R^{2560}$, matching the main-text notation.
The complete checkpoint and interface configuration is summarized in Appendix~\ref{app:model-config}.

For normalized box center $(c_x,c_y)$, width $w$, height $h$, and area $A_t=wh$, the geometric input is
\begin{equation}
\label{eq:geometry}
g_t=\left[c_x,c_y,w,h,A_t,d_t,\log\frac{A_t}{A_{t^-}},
\frac{t}{\max(T-1,1)}\right],\qquad
\widetilde a_t=a_t+\operatorname{MLP}_{\mathrm{geo}}(g_t).
\end{equation}
Here $T$ is the original video length, $t^-$ is the preceding sampled observation, and $d_t$ is the Euclidean displacement of normalized centers. Degenerate boxes are discarded, so areas are positive. The first observation has zero displacement and zero log-area change. Displacement is measured between observations and is not a time-normalized velocity; timestamps expose unequal sampling intervals to subsequent temporal modeling. The geometry MLP is $\operatorname{Linear}(8,H)\to\operatorname{GELU}\to\operatorname{Linear}(H,H)$. A shared, single-latent multi-head attention resampler with LayerNorm aggregates each observation block into $f_i$. The blocks retain chronological order before entering the language model.

\subsection{Slow branch and fusion}
The Slow branch independently runs the same frozen SigLIP2-Large-based vision tower initialized from \texttt{Qwen3-VL-4B-Instruct} at pixel budget $P_s=200{,}704$.
It produces full-frame token banks $X_t^s$ with $N_t$ spatial tokens.
The cross-attention projections accommodate the vision-encoder and object-projector feature widths. The matching Fast appearance feature $a_t$, before explicit geometry is added, supplies the cross-attention query in Equation~\ref{eq:fusion}. ROI extraction anchors the query to the tracked object, but does not guarantee distinct attention patterns for nearby objects with similar appearance. In particular, implicit positional information in visual features is not a substitute for a guaranteed geometric correspondence.

The Slow resampler pools the $m_s(r)$ retained context features into one vector $c$, without cluster-size reweighting. Pruning changes the temporal attention mass as well as the number of frames; Lemma~\ref{lem:dualres} explicitly includes this effect. Without temporal positional inputs this pooling is permutation-invariant. Accordingly, it represents aggregate scene context rather than temporally aligned interactions or event order. Fast features and timestamps carry the ordered evidence. Rapidly changing interactions that need high-resolution detail are a limitation of this division of labor.

Fusion uses $z_i=\operatorname{LN}(f_i+\alpha W_sc)$ for valid blocks only. The scalar $\alpha$ is initialized to zero, so the initial representation is Fast-only. At that initialization, gradients to the Slow residual's parameters through this gate are also zero; the gate can update first, after which the context pathway can receive gradients. Padded slots are set to zero after fusion and excluded from attention, rather than receiving the shared context residual.

\subsection{Language-model inputs and adaptation}
The shared language backbone is initialized from the language component of \texttt{Qwen3-VL-4B-Instruct}, with hidden width $H_{\mathrm{LM}}$. The visual slot embedding is $W_{\mathrm{in}}z_i+\delta_{\mathrm{step}}$, where $W_{\mathrm{in}}$ denotes LayerNorm followed by a linear projection from the fused feature width to $H_{\mathrm{LM}}$.
The two retrieval paths are instantiated as separate prompts, matching the input isolation described in the main text. The trajectory prompt is constructed from the valid trajectory steps for a candidate and the visual readout:
\begingroup\small
\begin{verbatim}
Trajectory steps: <trajectory_step> ... <trajectory_step>
Read the complete temporal representation at <action_visual_readout>.
\end{verbatim}
\endgroup
The text prompt is a separate text-only input containing the query expression and the text readout:
\begingroup\small
\begin{verbatim}
Match the expression to a trajectory: {expression}
Read the text representation at <action_text_readout>.
\end{verbatim}
\endgroup
The trajectory pass injects the $K_r$ valid fused trajectory embeddings into the trajectory-step positions and reads $h_v$ at the visual readout; any padding used for batching is masked outside the prompt. The text pass uses only the text prompt above and reads $h_e$ at the text readout. It does not instantiate valid trajectory-step inputs; if fixed-size visual-slot tensors are reserved by the implementation, they are zeroed and masked so that no trajectory evidence is active. Thus $q(e)$ depends only on the query and the text prompt. Input isolation is established by these separate passes and masks, not by causal attention alone.

The output maps $W_v$ and $W_t$ are independent bias-free linear projections followed by L2 normalization. The logit scale starts at $\ell=\log(1/0.05)$, so the initial inverse temperature is $\exp(\ell)=20$ and the temperature is $0.05$. A positive shared scale changes the training distribution but not cosine rankings for a fixed query.

The pretrained language-model weights and language head remain frozen. LoRA adapters of rank 8, scaling parameter 16, and dropout 0.05 are inserted in the query and value projections of the last eight layers. Trainable modules are the geometry MLP; Fast and Slow resamplers; Slow cross-attention, query projection, residual projection, and gate; input and output projections; step and readout-token deltas; the LoRA parameters; and the logit scale. The vision encoder and ROI projector are frozen.
Exact checkpoint identifiers, optimization settings, batch configuration, training duration, hardware, software versions, and the random seed are summarized in Appendix~\ref{app:experimental-details}.

\section{Training Data and Relation Annotations}
\label{app:training}

\subsection{Annotated positives and overlap-based screening}
\label{app:pos}
A query group is indexed by $(\mathrm{video\_id},\mathrm{expression\_id})$.
For each expression-centered training sample, the batch constructor targets four trajectory records.
One clean ground-truth trajectory and one confirmed negative are always retained.
The remaining two records are filled by either one noisy-GT trajectory and one identity-checked matched raw trajectory, or by two matched raw positives when the noisy-GT variant is unavailable.
Thus, a training group contains one clean GT, zero or one noisy GT, one or two matched raw positives, and one confirmed negative.

An annotation may cover only the query event: its observed bbox sequence is used as supervision without claiming boxes beyond its annotated interval.
A matched tracker trajectory retains its full available extent and is not cropped to that interval.
This distinction reconciles event-limited supervision with identity-track retrieval.

Let $G_b\ne\varnothing$ be the target's annotated visible-frame set and $a_t(r)$ its box IoU with candidate $r$, taking missing observations as zero. Compute metrics on original observations, before encoder sampling:
\begin{equation}
\label{eq:iou}
\operatorname{maxIoU}_b(r)=\max_{t\in G_b}a_t(r),\quad
C_\tau(r)=\frac{\sum_{t\in G_b}\mathbf1[a_t(r)\geq\tau]}{|G_b|},\quad
\operatorname{meanIoU}_b(r)=\frac{\sum_{t\in G_b}a_t(r)}{|G_b|}.
\end{equation}
Candidates with $C_{0.5}\geq0.8$ are ordered by descending coverage, descending mean IoU, and then track ID. Include the highest-ranked candidate that retains the target identity throughout and preserves the query's visual evidence, if one exists. Identity checking outside $G_b$ is necessary because overlap inside the annotated interval cannot rule out later identity switches.

\textbf{Noisy-GT trajectory augmentation.}
Noisy-GT augmentation is enabled by default.
Starting from the clean ground-truth trajectory, we first remove one contiguous segment comprising approximately 10\%-20\% of its observations and then independently drop each remaining observation with probability 0.1.
The resulting noisy trajectory is required to retain at least one valid observation.
The clean ground-truth trajectory is preserved independently as an unperturbed positive record.

For each retained box, the horizontal and vertical center coordinates are independently displaced by up to $\pm5\%$ of the original box width and height, respectively.
The box width and height are independently multiplied by scale factors in $[0.9,1.1]$.
Perturbed boxes are clipped to the image boundaries.
If clipping produces a degenerate box, the original box for that observation is restored.
The exact augmentation configuration is summarized in Appendix~\ref{app:augmentation-config}.

\textbf{Negative verification.}
The threshold $\operatorname{maxIoU}_b(r)<\tau_{\mathrm{safe}}=0.2$ proposes a low-overlap negative. It is not a semantic certificate: a same-person fragment outside $G_b$ can satisfy this inequality, and another person can satisfy an underspecified query. We therefore define $\mathcal N_b$ to contain low-overlap proposals confirmed not to match $e_b$. Exclude known same-identity fragments and unresolved cases. Tracks with intermediate overlap, uncertain identity, or inadequate evidence are ignored. Positive and negative memberships are query-relative and must never conflict for the same encoded input.

\subsection{Model-assisted, human-verified training data}
For accepted training examples without manually annotated target tubes, an LLM verbalizes the human-approved trajectory facts $D_r^H$ into a natural query $e_b$, following Appendix~\ref{app:benchmark-description}.
The query's required predicates must be supported by the approved facts, including their temporal scope.
Human verification against the complete within-video candidate pool determines query admission and supplies the target positive and confirmed within-video negatives, as described in Appendix~\ref{app:benchmark-relations}.
Cross-video candidates are mined as easy-negative candidates, hard-negative candidates, or potential matches according to Appendix~\ref{app:benchmark-cross}.
These mining categories and text similarities propose comparisons rather than determine relation labels.
Only individually human-verified matches and mismatches populate the core batch mask; unverified or unresolved pairs receive $M_{bp}=\bot$.
Calibrated automatic easy negatives, when used, are separately identified as training-only pseudo-labels and excluded from the human-verified core and benchmark ground truth.

Annotated and expanded examples can be mixed in training. A source contributes to the annotated branch only when it actually supplies object tubes or when such tubes have been separately annotated. Video captions or action-class labels alone do not provide target boxes. Thus videos sourced from UCF-101, MSR-VTT, or MSVD do not automatically become human tube-annotated examples; source-specific supervision and any added annotation must be documented. HC-STVG-style event annotations can provide tube supervision, but ranking complete proposed tracks is different from the original joint spatial and temporal grounding task.

\subsection{Batch mask, duplicate records, and empty comparisons}
The list $\mathcal P$ in Section~\ref{sec:mask} indexes records, not physical identities. An identical encoded trajectory may occur in several source groups; separately perturbed inputs remain distinct records. Every source query matches its own positive records. Other entries are assigned by the query-trajectory relation: propagate a positive label only when the other query is also supported by that trajectory's evidence, assign zero only to established mismatches, and otherwise abstain. Shared identity alone need not establish that an event description is visible in every partial trajectory.

Within a video, existing positive/negative annotations and verified query judgments provide these relations. Across different videos, box IoU is undefined as an identity test, and different video IDs do not prove semantic mismatch. An independently verified relation may be used; otherwise the entry remains $\bot$. Consequently the batch term supplies cross-video comparisons only where those labels exist. This is a limitation of conservative supervision, not an assumption that all off-diagonal pairs are negatives.

The numerator of $L_{t\to v}$ must include \emph{all records} satisfying $M_{bp}=1$, including matching records from other groups. Using only the source set $\mathcal P_b$ in that numerator while admitting additional matching records into its denominator would incorrectly penalize those records. Repeated records act as multiplicity weights in the log-sum-exp. For identical encoded inputs their scores are tied, and their derivative contributions add; their positions in different loss expressions do not make the underlying similarities independent.

Rows and columns contain a positive by construction, avoiding undefined logarithms. A row or column with positives but no confirmed negatives contributes zero. Likewise $L_{\mathrm{group}}=0$ for an empty $\mathcal B_N$. Ignored pairs have no direct score derivative in the corresponding masked objective, though their shared encoder parameters may still change through other pairs. The multi-positive objective rewards total positive probability and does not require uniform probability across positive variants.

\section{Text2Traj-Bench Construction Details}
\label{app:benchmark}

This appendix details Text2Traj-Bench's three-stage construction for object-level video grounding and relation supervision for within-video discrimination and bidirectional batch retrieval.
The stages follow Section~\ref{sec:text2traj}: Trajectory Curation, Verified Description, and Relation Construction.
Models propose candidate trajectories, descriptions, text queries, and query-trajectory relations, while human review determines the core annotations.
Source assignments are fixed before construction: duplicate clips and clips from the same recording remain together, including all derived text queries and trajectory variants.
Construction settings are selected on training and development sources.
Cross-video training relations never connect to held-out sources, including reserved sources of overlapping external evaluation datasets.

\subsection{Trajectory Curation}
\label{app:benchmark-curation}
\label{app:benchmark-consolidation}

\textbf{Proposal and quality review.}
Candidate trajectory proposals are generated independently of text queries by a frozen SAM~3 tracker using a shared category-prompt inventory rather than query-specific appearance prompts.
Each proposal retains its available bounding boxes and original timestamps, as defined in Appendix~\ref{app:candidates}.
The prompt inventory, checkpoint, anchor-frame policy, and tracking settings are versioned; ``default prompts'' alone do not specify a reproducible configuration.
Proposal recall is limited by this process.

Three independent annotators determine proposal usability from temporal playback and target-marked frames.
Observation count, temporal coverage, contiguous support, box area, and abrupt position or scale changes assist review rather than automatically determine acceptance.
Reviewers label each proposal usable, unusable, or uncertain based on object recognizability, identity continuity, tracking drift, and visual support.
At least two usable votes are required for retention.
Cases without a majority undergo adjudication and receive another recorded decision; a proposal without two approvals is not retained.
Small size, short visibility, rapid motion, and occlusion are not rejection criteria by themselves.
Individual votes and exclusion reasons are retained.

\textbf{Duplicate and fragment review.}
Duplicate and fragment review consolidates confirmed redundant trajectories while preserving distinct objects and nonredundant fragments as separate candidates.
Spatiotemporal overlap and visual similarity propose duplicate pairs.
All retained candidate trajectories are also displayed together to identify missed pairs.
Three annotators independently judge each pair as having the same identity, different identities, or uncertain identity.
At least two supporting votes are required to establish either identity decision.
Same identity alone does not make two trajectory records redundant: consolidation removes only confirmed redundant observations and retains the highest-quality representative with deterministic tie-breaking.
Nonredundant fragments remain separate without interpolating missing observations.
Inconsistent pairwise decisions are adjudicated before duplicate groups are formed; transitive merging must not override a confirmed different-identity judgment.

The candidate pool $\mathcal R(V)$ is frozen after curation and remains unchanged by subsequent description or query rejection.
Unresolved identity pairs remain separate candidates and are flagged for relation review rather than deleted to make text queries unique.
A query matching multiple fragments is rejected, whereas a query supported by only a particular fragment is assessed against that fragment's actual observations.
A video requires at least two candidate trajectories and one accepted text query for benchmark inclusion.
Subsequent description failure removes only a query proposal, not its candidate trajectory.
Correcting a subsequently discovered tracking defect requires a new candidate-pool version and renewed verification of all affected queries.

\subsection{Verified Description}
\label{app:benchmark-description}
\label{app:benchmark-evidence}

\textbf{Structured visual evidence.}
For each retained candidate trajectory, a VLM proposes descriptions and structured facts grounded in its ordered visual observations.
Its inputs contain chronologically ordered crops and selected temporal segments.
Target-marked full frames are additionally provided for spatial relations, interactions, and scene context.
The VLM returns separate appearance and motion descriptions together with structured facts recording the attribute or event, value, evidence frames or interval, temporal scope, and known or unknown status.
Appearance fields describe visible properties rather than inferred demographics.
Action entries distinguish event intervals and participants, while direction entries record their image-relative convention.
Unobserved intervals and occlusions are recorded rather than treated as evidence of absence.
Sampling and targeted evidence expansion are recorded together with their budgets.

\textbf{Three-person verification.}
Three independent annotators verify the proposed descriptions and structured facts against the video to establish the human-approved facts $D_r^H$.
They label each description usable, usable after minor correction, unusable, or uncertain.
Minor corrections may remove unsupported details, simplify an overstated action, or correct a clearly visible attribute; extensive rewriting requires a new proposal.
All edits are consolidated into a single final version, which must itself receive at least two approvals.
Agreement that different edits are individually acceptable does not establish approval of their union.
Unresolved fields remain unknown and cannot supply facts for text query generation.
The resulting $D_r^H$ retains approved statements and their supporting evidence rather than providing an exhaustive account of everything that occurred.
A rejected description does not remove its candidate trajectory from $\mathcal R(V)$.

\textbf{Text query generation.}
An LLM generates a natural text query $e$ and its required predicates $P(e)$ from $D_r^H$ alone, without access to competing candidate trajectories.
Each predicate must be entailed by the human-approved facts, including its temporal and relational qualifiers; lexical copying is unnecessary.
Facts from different intervals cannot be combined into an unsupported simultaneous event.
Text query generation uses a fixed, recorded proposal budget, and minor revisions are versioned.
The next stage checks query faithfulness and naturalness; any changed wording is rechecked against the complete candidate pool $\mathcal R(V)$.
These are model-generated text queries rather than independently collected user utterances.
More fluent paraphrasing alone does not establish coverage of the distribution of natural user queries.

\subsection{Relation Construction}

\subsubsection{Within-video Verification}
\label{app:benchmark-relations}

\textbf{Complete human verification.}
Within-video relation annotation verifies every proposed text query against the complete candidate pool $\mathcal R(V)$, including candidate trajectories without approved descriptions.
With $Q_V$ proposed queries and $K_V$ candidate trajectories, the relation matrix has shape $Q_V\times K_V$, not necessarily $K_V\times K_V$.
LLM pre-labels may highlight relevant facts, but a human annotator verifies every query-trajectory pair using the actual visual observations and available context.
Disputed decisions receive an independent second review, and relations that remain uncertain retain $\bot$.
Query faithfulness and naturalness are checked in the same pass.

During construction, relation annotation distinguishes confirmed matches, confirmed mismatches, and uncertain relations, consistent with the relation mask $M_{bp}$ in Section~\ref{sec:mask}.
We distinguish the LLM proposal $\widehat Y(e,r)$, the final annotation $\bar Y(e,r)$, and the unobserved true semantic relation $Y(e,r)$.
The labels $1$, $0$, and $\bot$ respectively denote a confirmed match, a confirmed mismatch, and an uncertain relation under the following policy:
\begin{equation}
\label{eq:relation-policy}
\bar Y(e,r)=
\begin{cases}
1, & \text{all required predicates are jointly supported},\\
0, & \text{a necessary condition is demonstrably violated},\\
\bot, & \text{otherwise}.
\end{cases}
\end{equation}

Both confirmed matches and confirmed mismatches must respect the text query's temporal scope and participant bindings.
Observing running in one interval does not refute ``walks at some point.''
A confirmed mismatch based on an absent event requires adequate coverage of the relevant interval; omitted or uncertain attributes do not establish a mismatch.
Human reviewers may inspect additional existing observations when the description is insufficient.
Joint support additionally requires consistent event timing and participant binding.

\textbf{Admission and within-video discrimination.}
Benchmark admission establishes a unique target trajectory and provides confirmed relations for within-video discrimination.
Equation~\ref{eq:unique} requires one confirmed target match and confirmed mismatches for every alternative candidate trajectory in $\mathcal R(V)$.
An additional confirmed match or any uncertain relation blocks admission without altering the candidate pool or other queries.
Consequently, uncertain labels can remain in construction logs and in training masks for unverified pairs, but admitted benchmark queries have binary within-video relation rows: one confirmed target match and confirmed mismatches for all alternatives.
This establishes an annotation-based decision within the candidate pool rather than infallible uniqueness in the physical scene.
Accepted training queries define the positives and confirmed within-video negatives of each query-centered trajectory group:
\begin{equation}
\label{eq:groups-from-relations}
\mathcal P_b=
\{r\in\mathcal R(V_b):\bar Y(e_b,r)=1\},
\qquad
\mathcal N_b=
\{r\in\mathcal R(V_b):\bar Y(e_b,r)=0\}.
\end{equation}
These sets supply the within-video discrimination objective $L_{\mathrm{group}}$.
Training-time trajectory augmentations are handled separately in Appendix~\ref{app:training}.
The core benchmark uses only accepted single-target text queries; rejected proposals and their reasons remain in construction logs rather than providing training supervision.
Every released held-out query has a fully human-reviewed row of within-video relations, not merely a sampled check of automatic labels.

\subsubsection{Cross-video Expansion}
\label{app:benchmark-cross}

\textbf{Stratified candidate mining.}
For bidirectional training, cross-video expansion mines additional query-trajectory comparisons exclusively from training queries and trajectories.
Each text query is compared with the candidate trajectory's facts rather than merely with another target's caption.
Mining combines structured compatibility checks with appearance and motion similarity to form three strata.
Easy-negative candidates exhibit obvious category or basic-property conflicts.
Hard-negative candidates have otherwise similar descriptions but a possible conflict in a necessary attribute or event.
Potential matches have descriptions that may support the text query.
These strata guide sampling rather than determine final relation labels, and similarity ranks candidate pairs only.
Category hierarchies, multicolored clothing, temporal changes, and concurrent actions prevent different field values from automatically establishing a confirmed mismatch.
Random cross-video query-trajectory pairs supplement similarity-based mining to audit its selection bias.
Mining budgets, random seeds, encoders, and stratum assignments are logged; duplicate pairs retain their provenance but are reviewed once.

\textbf{LLM proposal and human confirmation.}
Mined query-trajectory pairs receive LLM relation proposals, while individual human verification determines the confirmed relations in the core relation graph.
For each mined pair, the LLM receives $e$, $P(e)$, and $D_r^H$, and predicts $\widehat Y(e,r)\in\{1,0,\bot\}$ together with the decisive predicates and evidence references.
It applies Equation~\ref{eq:relation-policy} and assigns $\bot$ when the descriptions leave the relation unresolved, including unresolved event scope or joint support.
A textual omission does not establish a confirmed mismatch, and high similarity does not establish a confirmed match.
Potential matches and hard-negative candidates receive pairwise human review with access to visual evidence; disagreements prompt independent adjudication.
Easy-negative candidates are reviewed in stratified samples.
Individually confirmed pairs enter the core relation graph, while unverified pairs remain $\bot$.
Mining strata and final relation labels are stored separately, allowing a mined hard-negative candidate to become a confirmed match or an uncertain relation.

An optional training-only expansion uses calibrated automatic easy negatives as pseudo-labels separately from the human-verified core.
Automatic easy negatives are accepted only after independent human calibration on development pairs selected by the same rule.
Acceptance uses a predeclared upper bound on error with a confidence interval rather than LLM self-confidence.
Uncertain audit outcomes count against acceptance.
Rule-specific sample sizes, errors, coverage, and drift checks must be reported, and the rule is frozen before application.
These relations are tagged as pseudo-labels, excluded from benchmark ground truth, and evaluated separately from the human-verified core.
No fixed numerical accuracy or cost saving is assumed.

\textbf{From verified relations to bidirectional batch retrieval.}
Confirmed within-video and cross-video relations populate the relation mask $M_{bp}$ for bidirectional batch retrieval when the corresponding text queries and trajectory records occur in the same batch.
To match Section~\ref{sec:mask}, mining for this objective targets candidate trajectories with an accepted source query, making them eligible for the indexed list of positive records $\mathcal P$.
Relation-aware batching co-selects linked query-centered trajectory groups so that their verified relations are available for the row and column comparisons of $L_{\mathrm{bi}}$.
A confirmed mismatch involving a trajectory absent from $\mathcal P$ supplies no comparison for $L_{\mathrm{bi}}$.
Missing relations are assigned $\bot$ rather than treated as confirmed mismatches, and each row and column retains its source positive.
Reporting must include realized cross-video positive and negative counts and zero-negative row and column rates, not only the size of the offline relation graph.
A positive-only comparison contributes zero loss.
Transferring a relation label to an augmented trajectory record requires checking that the evidence deciding that particular relation is preserved; identity agreement alone is insufficient.

\paragraph{Shared auditing and release.}
\label{app:benchmark-audit}

Independent auditing and traceable release records document annotation reliability and construction provenance.
Independent auditors are blinded to model suggestions and previous votes and inspect a random sample of accepted records together with a separately reported sample of difficult cases.
The audit records residual errors, pre-adjudication agreement, three-vote distributions, corrections, query rejection reasons, and reviewer time per accepted query or relation.
Text query and candidate-trajectory manifests store the queries, candidate and target trajectory IDs, original trajectory observations, source and split assignments, and candidate-pool versions.
Linked fact, vote, and relation records retain evidence, correction history, mining strata, provenance, and pseudo-label status.
Prompts, checkpoints, manifests, budgets, and audit statistics must be supplied from actual construction runs.
Majority approval is an annotation rule rather than a guarantee of correctness.

\paragraph{Evaluation scope.}
\label{app:benchmark-evaluation}

Evaluation measures object-level video grounding through similarity-based ranking of candidate trajectories within each video's fixed candidate pool $\mathcal R(V)$.
Models receive the text query $e$, video $V$, and candidate trajectories in $\mathcal R(V)$; target trajectory IDs, verified descriptions, relation labels, and annotation evidence tags are not model inputs.
Every method ranks the same complete candidate pool for each text query.
Reporting includes within-video R@1, candidate-count strata $2$, $3$-$5$, $6$-$10$, and $>10$, and a random-ranking baseline obtained by averaging $1/|\mathcal R(V)|$ over text queries.
Single-frame and motion-dependent analyses examine whether temporal evidence is necessary, separately from whether a unique target trajectory is established.

Beyond ranking within the candidate pool, an independently sampled source subset is required to estimate proposal recall, post-curation retention, and end-to-end success.
The admission and selection events satisfy
\begin{equation}
\label{eq:end-to-end}
\Pr(S\cap A)=\Pr(A)\Pr(S\mid A),
\end{equation}
where $A$ denotes benchmark admission and $S$ denotes correct target-trajectory selection.
Both factors must be measured on the same source subset; within-video R@1 on the screened benchmark cannot replace the conditional accuracy measured on that subset.
Reverse, cross-video, and global retrieval require separate galleries and relevance protocols and are not established by within-video ranking scores.

\section{Lemma 1: Task-Relevant Preservation under Slow-Fast Sampling}
\label{app:proof1}

\subsection{Reference representations and semantic coverage}
Fix a trajectory, its Fast observation set $\mathcal F$, and all trained parameters. The analysis concerns deterministic encoding, with dropout disabled. The Slow branch first obtains $S_0\subseteq\mathcal F$ by uniform index preselection and then greedily prunes it to $S$, as specified in Equation~\ref{eq:slow-selection}. Write $n_s=|S_0|$ and $m_s=|S|$. A representative $\pi(i)\in S$ is assigned to every $i\in S_0$, including $\pi(j)=j$ for retained frames. Since retained representatives are never deleted, all assignments satisfy the threshold in the final set.

For each potential Slow frame, $u_i$ denotes the object-conditioned context from Equation~\ref{eq:fusion}, evaluated using the same Fast ROI query and the same high-resolution encoder. The features of frames excluded at runtime are reference quantities for analysis, not additional computation in the deployed model. Let
\begin{equation}
\label{eq:context-references}
U_0=[u_i]_{i\in S_0},\quad
\widetilde U_0=[u_{\pi(i)}]_{i\in S_0},\quad
U_S=[u_j]_{j\in S},\quad
c_0=R(U_0),\quad c=R(U_S).
\end{equation}
The proxy list $\widetilde U_0$ has the same length as $U_0$ and repeats each retained representative according to its cluster size. It is not the actual pruned input. The dense-context reference is $c_D=R([u_t]_{t\in\mathcal F})$, which would encode high-resolution context at every Fast observation. The temporal resampler summarizes a global context vector; the reference is not a reconstruction of temporally ordered scene interactions. The Fast temporal sequence is fixed throughout this comparison.

\textbf{Conditions.}
The lemma uses three explicit conditions. First, on assigned pairs, CLIP-style auxiliary semantic similarity controls object-conditioned context difference up to a residual:
\begin{equation}
\label{eq:clip-context-condition}
\|u_i-u_{\pi(i)}\|_2\leq L_u\|\bar g_i-\bar g_{\pi(i)}\|_2+\eta,
\qquad L_u\geq0,\quad\eta\geq0.
\end{equation}
Second, at the preselected input length, the resampler satisfies
\begin{equation}
\label{eq:resampler-stability}
\|R(U)-R(U')\|_2\leq K_R\max_i\|u_i-u'_i\|_2
\end{equation}
on the feature domain of interest. Third, preselection approximates the dense summary with $\|c_0-c_D\|_2\leq\varepsilon_{\mathrm{pre}}$. This last quantity includes the effects of finite coverage and changing the number of inputs to attention. The conditions do not assert that auxiliary semantic-image features are sufficient for every object interaction or that sparse pooling is automatically accurate.

The first condition is weaker than equating auxiliary semantic similarity with context equality: the residual $\eta$ allows discrepancies between global scene similarity and the object's contextual state. A large $\eta$ makes the resulting guarantee weak, as it should. The second condition holds with suitable finite constants for standard fixed-query attention on bounded feature domains, with finite projection norms and stable postprocessing. Output normalization additionally requires a positive lower bound on the pre-normalization embedding norm. These are local stability conditions to justify or measure; the proof does not assume that arbitrary learned representations have small sensitivity constants.

\subsection{Proof of semantic-pruning and representation bounds}
\textbf{Restatement of Lemma~\ref{lem:dualres}.}
Under the conditions above, define $\rho_\tau=L_u\sqrt{2(1-\tau)}+\eta$ and $\varepsilon_{\mathrm{mult}}=\|R(\widetilde U_0)-R(U_S)\|_2$. The context errors satisfy
\begin{equation}
\label{eq:context-pruning-bound}
\|c-c_0\|_2\leq K_R\rho_\tau+\varepsilon_{\mathrm{mult}},\qquad
\|c-c_D\|_2\leq\varepsilon_{\mathrm{pre}}+K_R\rho_\tau+\varepsilon_{\mathrm{mult}}.
\end{equation}

\begin{proof}
Unit normalization and the representative condition imply
\[
\|\bar g_i-\bar g_{\pi(i)}\|_2^2
=2-2\bar g_i^\top\bar g_{\pi(i)}\leq2(1-\tau).
\]
Equation~\ref{eq:clip-context-condition} gives $\max_i\|u_i-u_{\pi(i)}\|_2\leq\rho_\tau$. The original and proxy lists have equal length, so Equation~\ref{eq:resampler-stability} yields $\|R(U_0)-R(\widetilde U_0)\|_2\leq K_R\rho_\tau$. A triangle inequality through $R(\widetilde U_0)$ proves the first bound. Another triangle inequality through $c_0$, using the preselection condition, proves the second. The multiplicity change is handled separately because fixed-length stability does not apply directly to lists of different lengths.
\end{proof}

Let $F_r=(f_1,\ldots,f_{K_r})$ be the unchanged Fast sequence, and let $\Phi(F_r,c)$ denote fusion, language-model aggregation, projection, and L2 normalization. For the fixed candidate assume
\[
\|\Phi(F_r,c)-\Phi(F_r,c')\|_2\leq K_c\|c-c'\|_2.
\]
Applying this to the second context bound proves Equation~\ref{eq:slowfast-preservation}, with $v_D(r)=\Phi(F_r,c_D)$. The constants and errors may depend on the candidate; the ranking argument uses their maximum over the fixed candidate pool.

\textbf{Retrieval consequence.}
For a unit text embedding $q$, Cauchy-Schwarz gives $|q^\top v(r)-q^\top v_D(r)|\leq\delta_r$. Suppose $r^*$ is correctly ranked first by the reference, with cosine margin
\begin{equation}
\label{eq:reference-margin}
\gamma=q^\top v_D(r^*)-\max_{r\ne r^*}q^\top v_D(r)>0.
\end{equation}
Writing $\delta=\max_r\delta_r$, every competitor satisfies
\[
q^\top v(r^*)-q^\top v(r)
\geq q^\top v_D(r^*)-q^\top v_D(r)-\delta_{r^*}-\delta_r
\geq\gamma-2\delta.
\]
Thus $\gamma>2\delta$ preserves the correct top-1 decision. The positive logit scale does not change this ranking statement. The reference must itself be correct; no correctness guarantee is inferred merely from using more high-resolution frames.

\subsection{Attention mass after deleting repeated representatives}
Semantic closeness alone does not control the full effect of pruning. The proxy input repeats representative $j$ exactly $\kappa_j=|\{i:\pi(i)=j\}|$ times, while the actual input includes it once. For a fixed-query attention head $h$, let $z_{hj}$ be its query-key logit and $v_{hj}$ its projected value for retained feature $u_j$. The two outputs can be written using distributions on the same representative set:
\begin{equation}
\label{eq:attention-masses}
p_{hj}=\frac{\kappa_je^{z_{hj}}}{\sum_{k\in S}\kappa_ke^{z_{hk}}},\qquad
q_{hj}=\frac{e^{z_{hj}}}{\sum_{k\in S}e^{z_{hk}}},\qquad
y_h^{\mathrm{proxy}}=\sum_jp_{hj}v_{hj},\quad
y_h^{\mathrm{pruned}}=\sum_jq_{hj}v_{hj}.
\end{equation}
These formulas apply when the resampler has no per-frame positional or temporal bias that distinguishes repeated identical inputs, as in the order-invariant Slow pooling described in Appendix~\ref{app:encoding}. If $\|v_{hj}\|_2\leq V_h$, then
\begin{equation}
\label{eq:head-multiplicity-bound}
\|y_h^{\mathrm{proxy}}-y_h^{\mathrm{pruned}}\|_2
\leq\sum_j|p_{hj}-q_{hj}|\|v_{hj}\|_2
\leq V_h\|p_h-q_h\|_1.
\end{equation}
If the output projection and subsequent postprocessing together are $K_{\mathrm{post}}$-Lipschitz on concatenated head outputs, a computable upper bound is
\begin{equation}
\label{eq:multiplicity-bound}
\varepsilon_{\mathrm{mult}}\leq
K_{\mathrm{post}}\left(\sum_hV_h^2\|p_h-q_h\|_1^2\right)^{1/2}.
\end{equation}
For multiple fixed latent queries, apply the same calculation to each head-query pair and to the final aggregation. The bound separates perturbing a feature from changing the frequency with which its representative enters attention. Cluster counts and retained-frame attention logits determine the mass term without requiring high-resolution features for the deleted frames.

\textbf{Why this term cannot be omitted.}
Consider scalar values consisting of three copies of 0 and one copy of 1, with equal attention logits. Choose auxiliary semantic-image features identical within each value class and different between classes. Pruning exact duplicates leaves values 0 and 1. Even with zero within-cluster feature distortion, the unpruned output is $0.25$ and the pruned output is $0.5$. A bound depending only on the auxiliary-feature covering radius could incorrectly predict zero error. In contrast, equal cluster sizes give $p_h=q_h$ and zero multiplicity error. Adding $\log \kappa_j$ to the retained attention logits would also reproduce proxy multiplicities, but that would be a modified pooling rule; it is not assumed or introduced in the current encoder.

\subsection{Preselection coverage, stability, and Fast features}
The preselection error is distinct from the pruning error. To make this explicit, normalize original frame indices and assign each Fast observation $t\in\mathcal F$ to a nearest preselected observation $\sigma(t)\in S_0$. Define the actual temporal covering radius $h_0=\max_{t\in\mathcal F}|t-\sigma(t)|$. Uniform selection in the valid-observation list does not imply a fixed physical-time radius, particularly across occlusions.

If object-conditioned contexts obey $\|u_t-u_{\sigma(t)}\|_2\leq L_th_0$, use the repeated preselection proxy $[u_{\sigma(t)}]_{t\in\mathcal F}$. The same proof gives
\begin{equation}
\label{eq:preselection-bound}
\|c_0-c_D\|_2\leq
K_R^{D}L_th_0+\varepsilon_{\mathrm{pre,mult}},
\end{equation}
where $K_R^{D}$ is the fixed-length stability constant at $|\mathcal F|$ inputs, and $\varepsilon_{\mathrm{pre,mult}}$ is the change from the repeated proxy to the unweighted preselected input. Equation~\ref{eq:multiplicity-bound} applies with the preselection assignment counts. This provides one sufficient choice of $\varepsilon_{\mathrm{pre}}$ and shows why slow variation helps, while not silently assuming that changing attention multiplicities is harmless. A small time covering radius alone does not eliminate the mass term.

For completeness, a fixed-query attention head has a finite feature-perturbation constant on bounded inputs. Suppose corresponding keys and values differ by at most $L_k d$ and $L_v d$, value norms are bounded by $V$, and the fixed query has norm at most $Q$. For key dimension $d_k$, logits differ by at most $QL_kd/\sqrt{d_k}$. The softmax Jacobian gives $\|\operatorname{softmax}(z)-\operatorname{softmax}(z')\|_1\leq2\|z-z'\|_\infty$, so the head-output difference is at most
\[
\left(L_v+\frac{2VQL_k}{\sqrt{d_k}}\right)d.
\]
Combining heads and the Lipschitz postprocessing map yields a valid $K_R$. This argument needs bounded projections and a stable feature domain, and does not claim that the constants are small in every trained model.

Fusion also exposes the context sensitivity explicitly. If LayerNorm is $K_{\mathrm{LN}}$-Lipschitz on valid block inputs, the stack of fused blocks changes by at most
\[
\sqrt{K_r}\,K_{\mathrm{LN}}|\alpha|\|W_s\|_2\,\|c-c'\|_2
\]
in Frobenius norm. Subsequent input projection, backbone aggregation, and normalized readout contribute their own stability factors to $K_c$. A positive lower bound on the unnormalized readout norm is needed for the last normalization. At $\alpha=0$, the representation is independent of context, consistent with the initialization. The lemma concerns fixed parameters and does not imply that the gate or the representation error changes monotonically during training.

The reference in Lemma~\ref{lem:dualres} uses the same low-resolution Fast sequence and a denser Slow context. Thus the theorem specifically justifies context compression while keeping the motion observations unchanged. If a higher-quality Fast reference $F_r^*$ is also considered, define $\varepsilon_F=\|F_r-F_r^*\|$ and assume joint stability
\[
\|\Phi(F_r,c)-\Phi(F_r^*,c_D)\|_2
\leq K_F\varepsilon_F+K_c\|c-c_D\|_2.
\]
The total representation bound then adds $K_F\varepsilon_F$ to Equation~\ref{eq:slowfast-preservation}. This isolates rather than assumes away the effect of low Fast resolution. The reference $\Phi(F_r^*,c_D)$ fixes the same dense context; comparison to a system that also changes the ROI queries or context encoder needs to include those differences in its context error. Tiny objects and fast, detailed interactions can violate the small-error regime.

\subsection{Encoding cost and scope of the design argument}
Let $C_s$ be the cost of one high-resolution full-frame encoder pass under a common frame budget. Preselection followed by pruning reduces the number of Slow passes from $n_s$ to $m_s$, saving $(n_s-m_s)C_s$ before accounting for pruning overhead. If $C_{\mathrm{aux}}$ and $C_{\mathrm{prune}}$ are additional costs relative to the unpruned baseline, a sufficient condition for a saving in this portion of the pipeline is
\begin{equation}
\label{eq:pruning-cost}
(n_s-m_s)C_s>C_{\mathrm{aux}}+C_{\mathrm{prune}}.
\end{equation}
If auxiliary semantic-image features are shared or cached, only their incremental cost belongs on the right. Actual frame sizes, resampling overhead, ROI reuse, and caching across candidates must be included in measured end-to-end cost. No positive saving is guaranteed when no frame is removed or when the additional work dominates.

Ignoring auxiliary modules, the visual-encoder cost proxy is $n_fP_f+m_sP_s$, with $P_f=65{,}536$ and $P_s=200{,}704$. The Fast frame limit remains 32; the Slow preselection budget is one to four, and the actual retained count can be smaller. These are implementation budgets, not information saturation points. The design allocates temporal coverage to the Fast branch and uses the pruning tolerance to bound one source of error in sparse high-resolution context. It does not establish strict dominance over every uniform-resolution encoder, lossless compression of all scene semantics, or a generally optimal threshold. The bridge residual, attention mass discrepancy, preselection error, reference margins, and actual runtime require validation on representative data. Uniform preselection followed by auxiliary-encoder pruning does not by itself guarantee correct object tracking or recall of an event absent from the observations.

\section{Lemma 2: Task-Aligned Multi-Positive Ranking Control}
\label{app:proof2}

\subsection{A common retrieval objective}
Each local, text-to-trajectory, or trajectory-to-text term in Section~\ref{sec:mask} is a comparison $A$ with a nonempty positive index set $A^+$ and a disjoint negative index set $A^-$. Indices refer to the records actually included in that comparison; repeated records retain their multiplicities. All scores are finite. Write
\begin{equation}
\label{eq:generic-loss}
Z_A^+=\sum_{j\in A^+}e^{s_j},\qquad
Z_A^-=\sum_{j\in A^-}e^{s_j},\qquad
\mathcal L_A=-\log\frac{Z_A^+}{Z_A^++Z_A^-}
=\log\left(1+\frac{Z_A^-}{Z_A^+}\right).
\end{equation}
The fraction is the total softmax probability of the correct answer set. This objective marginalizes over valid answers instead of arbitrarily selecting a single positive and penalizing the others. Its retrieval interpretation is that the highest-ranked answer should belong to that set; it does not require a uniform distribution over positives or that every positive outrank every negative.

For $m_A=|A^+|\geq1$, define the conservative top-1 error
\begin{equation}
\label{eq:ranking-error}
E_A=\begin{cases}
\mathbf1\!\left[\max_{j\in A^-}s_j\geq\max_{j\in A^+}s_j\right],&A^-\ne\varnothing,\\
0,&A^-=\varnothing.
\end{cases}
\end{equation}
Ties between a highest-scoring positive and negative count as errors, so the bound also applies to implementations that sometimes break such ties in favor of a positive. Its semantic interpretation assumes that the included labels correctly identify matching and non-matching inputs. No claim is made about excluded, unconfirmed pairs.

\subsection{Proof of the per-comparison bound}
\textbf{Restatement of Lemma~\ref{lem:dualoss}, first part.}
For every comparison above, $\log(1+1/m_A)E_A\leq\mathcal L_A$. In particular, $\mathcal L_A<\log(1+1/m_A)$ guarantees that the best positive strictly outranks every included negative.

\begin{proof}
If $A^-$ is empty, both $E_A$ and $\mathcal L_A$ are zero. If $E_A=0$, the claim follows from $\mathcal L_A\geq0$. Otherwise, set $s^+=\max_{j\in A^+}s_j$. A negative has score at least $s^+$, while each of the $m_A$ positives has score at most $s^+$. Hence
\begin{equation}
\label{eq:ranking-proof}
Z_A^-\geq e^{s^+},\qquad Z_A^+\leq m_Ae^{s^+},\qquad
\mathcal L_A\geq\log\left(1+\frac{1}{m_A}\right).
\end{equation}
This is the desired bound when $E_A=1$, completing all cases. The strict-loss threshold follows by contraposition.
\end{proof}

\textbf{Sharpness and multiple positives.}
The constant cannot be increased without additional assumptions: take $m_A$ positives and one negative, all with score zero. Then $E_A=1$ and $\mathcal L_A=\log(1+1/m_A)$. For a single positive, the threshold is $\log2$. With multiple positives the constant becomes smaller, reflecting that aggregated positive probability can be large even when one negative ties or exceeds each individual positive. Exact duplicate records must be counted in $m_A$ because they also appear repeatedly in the implemented sums; counting only distinct identities would not justify this bound for the current objective. The guarantee concerns the best valid answer, not alignment of every positive variant.

\subsection{Joint control of the three retrieval families}
Let $\mathcal F_g=\mathcal B_N$ index local comparisons, $\mathcal F_t=\mathcal B$ index batch rows, and $\mathcal F_v=\mathcal P$ index batch columns. Their positive and negative sets are
\begin{align}
\label{eq:comparison-families}
A_{g,b}^+&=\mathcal P_b,& A_{g,b}^-&=\mathcal N_b,\\
A_{t,b}^+&=\mathcal J_b^+,& A_{t,b}^-&=\{p\in\mathcal P:M_{bp}=0\},\nonumber\\
A_{v,p}^+&=\mathcal I_p^+,& A_{v,p}^-&=\{b\in\mathcal B:M_{bp}=0\}.\nonumber
\end{align}
For a nonempty family $\mathcal F_x$, define
\begin{equation}
\label{eq:family-risk}
E_x=\frac{1}{|\mathcal F_x|}\sum_{A\in\mathcal F_x}E_A,\qquad
m_x=\max_{A\in\mathcal F_x}|A^+|,\qquad \omega_x=\log(1+1/m_x).
\end{equation}
If a family is empty, set its error and loss to zero and use $m_x=1$. For example, this covers a batch with no local negatives. A row or column with no confirmed negatives likewise has zero loss and zero included-comparison error; it imposes no discrimination constraint.

\textbf{Restatement of Lemma~\ref{lem:dualoss}, second part.}
With $\lambda_g,\lambda_{\mathrm{bi}}>0$, the objective
\[
L=\lambda_g L_{\mathrm{group}}+
\frac{\lambda_{\mathrm{bi}}}{2}(L_{t\to v}+L_{v\to t})
\]
satisfies Equation~\ref{eq:joint-ranking-control}. This is a deterministic bound on the empirical errors for the included comparisons; it does not require independent comparisons or disjoint score pairs.

\begin{proof}
For any $A\in\mathcal F_x$, $m_A\leq m_x$, so $\omega_x\leq\log(1+1/m_A)$. The first part of the lemma gives $\omega_xE_A\leq\mathcal L_A$. Averaging over each family yields
\begin{equation}
\label{eq:family-bounds}
\omega_gE_g\leq L_{\mathrm{group}},\qquad
\omega_tE_t\leq L_{t\to v},\qquad
\omega_vE_v\leq L_{v\to t}.
\end{equation}
The definitions make these inequalities valid for empty families as well. Multiplying by $\lambda_g$, $\lambda_{\mathrm{bi}}/2$, and $\lambda_{\mathrm{bi}}/2$, respectively, and adding proves the joint bound. Score sharing between comparisons does not affect any inequality.
\end{proof}

\textbf{Implication for loss design.}
The local term controls object selection among confirmed candidates from the same video. The batch row term controls text-to-trajectory matching over the verified batch relations, while the column term controls the reverse selection of a valid description for each trajectory. A local negative absent from the batch's positive-record list receives direct supervision only through the local term. Conversely, a verified cross-video mismatch absent from a query's local pool can enter the batch objective. These are conditional sources of additional comparisons, not a claim that every batch supplies both kinds. The reverse direction also changes the conditional normalization even when the underlying score pairs overlap. Thus the joint objective accounts for the intended retrieval decisions without requiring disjoint negative sets.

\subsection{Relation masking and direct supervision}
The positive numerator must contain every included record known to match the query. A match erroneously included only in the denominator is treated as a negative, breaking the intended semantic interpretation of the error bound. Unknown relations are excluded from both sets; different video IDs alone do not certify a mismatch. Algebraically the lemma applies to the supplied labels, while its interpretation as control of correct semantic retrieval depends on those labels being accurate.

The derivative calculation provides an additional consistency check. With $Z_A=Z_A^++Z_A^-$,
\begin{equation}
\label{eq:generic-gradient}
\frac{\partial\mathcal L_A}{\partial s_j}=
\begin{cases}
e^{s_j}/Z_A>0,&j\in A^-,\\
e^{s_j}(1/Z_A-1/Z_A^+)\leq0,&j\in A^+,\\
0,&j\notin A^+\cup A^-.
\end{cases}
\end{equation}
The positive-pair derivative is strictly negative exactly when that comparison has a negative. For a batch mismatch $M_{bp}=0$, let $D_b=\sum_{j\in\mathcal J_b}e^{s_{bj}}$ and $C_p=\sum_{a\in\mathcal I_p}e^{s_{ap}}$. Its derivative in the bidirectional term is
\begin{equation}
\label{eq:batch-negative-gradient}
\frac{\partial L_{\mathrm{bi}}}{\partial s_{bp}}
=\frac{e^{s_{bp}}}{2}\left(\frac{1}{|\mathcal B|D_b}+
\frac{1}{|\mathcal P|C_p}\right)>0.
\end{equation}
Matching pairs have nonpositive contributions and ignored pairs have zero direct contributions. Identical query-input pairs have tied scores: the derivative with respect to that shared score sums the occurrence derivatives. Consistent labels therefore give sign-consistent contributions even when a score appears in both losses. Distinct loss indices do not make the underlying similarities independent.

\textbf{Scope of the justification.}
The bound supports the form of a multi-positive objective and the joint treatment of local and bidirectional retrieval. It does not establish an optimal ratio of $\lambda_g$ to $\lambda_{\mathrm{bi}}$; the adopted $0.1:1$ weighting requires empirical validation. A small average loss controls average included-comparison error, not necessarily every individual comparison. The result neither bounds errors on unconfirmed or unseen candidates nor accounts for tracker misses, annotation mistakes, or train-test distribution shifts. It also does not guarantee that the three empirical errors decrease monotonically during training. Shared encoder parameters couple score changes: for a small gradient step,
\[
\Delta s_i\approx-\eta\sum_j\frac{\partial L}{\partial s_j}
\langle\nabla_\theta s_i,\nabla_\theta s_j\rangle.
\]
Consequently the direct derivative signs are not a guarantee on every score update. Loss/mask ablations, relation-quality audits, and held-out retrieval evaluation remain necessary to establish the empirical benefit.

\end{document}